\documentclass{article}
\usepackage{graphicx}    
\usepackage{mathtools, nccmath}

\PassOptionsToPackage{sort&compress,numbers}{natbib}
\usepackage[preprint]{neurips_2021}
\usepackage[utf8]{inputenc} %
\usepackage[T1]{fontenc}    %
\usepackage[pagebackref=true,breaklinks=true,colorlinks,bookmarks=false]{hyperref}
\usepackage{url}            %
\usepackage{booktabs}       %
\usepackage{amsfonts}       %
\usepackage{nicefrac}       %
\usepackage{microtype}      %
\usepackage{makecell}

\usepackage{amsthm,amsmath,amssymb,tikz,stmaryrd}

\usepackage[noabbrev,capitalize]{cleveref}

\usepackage{versions}
\usepackage{subcaption}
\usepackage{amsthm}
\usepackage{lipsum}
\usepackage{caption}
\usepackage{amsfonts}
\usepackage{graphicx}
\usepackage{epstopdf}
\usepackage{wrapfig}
\usepackage{hyperref}
\usepackage{multicol}
\usepackage{listings}
\usepackage{amssymb}
\usepackage{array} %
\usepackage{algorithm}
\usepackage{algpseudocode}
\usepackage{xcolor}
\usepackage{colortbl}

\definecolor{commentgreen}{RGB}{74,112,35}
\newtheorem{fact}{Fact}

\DeclareMathOperator*{\argmin}{arg\,min}

\DeclareMathOperator*{\diag}{diag}
\newcommand{\norm}[1]{\left\lVert#1\right\rVert}

\newcommand{\bE}{\mathbb{E}}
\newcommand{\bP}{\mathbb{P}}

\newcommand{\bR}{\mathbb{R}}

\newcommand{\cC}{\mathcal{C}}
\newcommand{\cD}{\mathcal{D}}

\newcommand{\cL}{\mathcal{L}}

\newcommand{\cO}{\mathcal{O}}
\newcommand{\cP}{\mathcal{P}}

\newcommand{\cR}{\mathcal{R}}

\newcommand{\cW}{\mathcal{W}}
\newcommand{\cX}{\mathcal{X}}
\newcommand{\cY}{\mathcal{Y}}

\newcommand{\ffi}{\mathbf{i}}

\newcommand{\falp}{\pmb{\alpha}}
\newcommand{\fbet}{\pmb{\beta}}
\newcommand{\flam}{\pmb{\lambda}}

\newcommand{\fG}{\mathbf{G}}
\newcommand{\fH}{\mathbf{H}}

\usepackage[framemethod=TikZ]{mdframed}
\mdfsetup{frametitlealignment=\center}

\newcommand{\fracpartial}[2]{\frac{\partial #1}{\partial  #2}}
\newcommand{\mixup}[2]{\tilde{#1}_{#2}^{(\text{mixup})}}
\newcommand{\cutmix}[2]{\tilde{#1}_{#2}^{(\text{cutmix})}}
\newcommand{\msda}[2]{\tilde{#1}_{#2}^{(\text{MSDA})}}
\newtheorem{theorem}{Theorem}

\newtheorem{lemma}{Lemma}

\theoremstyle{definition}
\newtheorem{definition}{Definition}[section]
\usepackage{amsmath,amssymb,tikz,stmaryrd}
\usepackage{tikz}

\usepackage{times,amsmath,graphicx,amssymb}
\usepackage{tikz-3dplot}
\usetikzlibrary{shapes,calc,positioning}
\usetikzlibrary{calc,intersections}
\usepackage{tkz-euclide}

\usepackage{amsopn}
\usepackage{lipsum}

\definecolor{lightgrey}{gray}{0.8}
\definecolor{medgrey}{gray}{0.6}
\definecolor{darkgrey}{gray}{0.4}
\definecolor{Gray}{gray}{0.9}
\definecolor{darkergreen}{RGB}{21, 152, 56}
\definecolor{red2}{RGB}{252, 54, 65}

\usepackage{xspace}
\def\onedot{.\xspace}
\def\eg{\emph{e.g}\onedot} 
\def\ie{\emph{i.e}\onedot}

\def\etal{\emph{et al}\onedot}

\newcommand{\ours}{HMix\xspace}
\newcommand{\oursfull}{Hybrid version of Mixup and CutMix\xspace}
\newcommand{\oursgaussian}{GMix\xspace}
\newcommand{\oursgaussianfull}{Gaussian Mixup\xspace}
\newcommand{\stochmixupcutmix}{Stochastic Mixup \& CutMix\xspace}

\newcommand{\revision}[1]{#1}

\title{A Unified Analysis of Mixed Sample Data Augmentation: A Loss Function Perspective}

\author{
  Chanwoo Park\,\thanks{Equal contribution} %
  \\
  {MIT}\\
  {~~~~~~cpark97@mit.edu~~~~~~~~} 
   \And
 Sangdoo Yun\,\footnotemark[1]%
  \\
  {NAVER AI Lab}\\
  {sangdoo.yun@navercorp.com} 
   \And
 Sanghyuk Chun%
  \\
  {NAVER AI Lab}\\
  {sanghyuk.c@navercorp.com}   
}

\begin{document}

\maketitle

\begin{abstract}
We propose the first unified theoretical analysis of mixed sample data augmentation (MSDA), such as Mixup and CutMix. Our theoretical results show that regardless of the choice of the mixing strategy, MSDA behaves as a pixel-level regularization of the underlying training loss and a regularization of the first layer parameters. Similarly, our theoretical results support that the MSDA training strategy can improve adversarial robustness and generalization compared to the vanilla training strategy. Using the theoretical results, we provide a high-level understanding of how different design choices of MSDA work differently. For example, we show that the most popular MSDA methods, Mixup and CutMix, behave differently, \eg, CutMix regularizes the input gradients by pixel distances, while Mixup regularizes the input gradients regardless of pixel distances. Our theoretical results also show that the optimal MSDA strategy depends on tasks, datasets, or model parameters. From these observations, we propose generalized MSDAs, a Hybrid version of Mixup and CutMix  (\textbf{\ours}) and Gaussian Mixup (\textbf{\oursgaussian}), simple extensions of Mixup and CutMix. Our implementation can leverage the advantages of Mixup and CutMix, while our implementation is very efficient, and the computation cost is almost neglectable as Mixup and CutMix. Our empirical study shows that our \ours and \oursgaussian outperform the previous state-of-the-art MSDA methods in CIFAR-100 and ImageNet classification tasks. Source code is available at {\small \url{https://github.com/naver-ai/hmix-gmix}}.
\end{abstract}

\section{Introduction}
As deep neural networks (DNNs) are data-hungry, the scale of datasets has become a foundation of modern DNN training; recent ground-breaking deep models are built upon gigantic datasets, such as 410B language tokens \cite{gpt3}, 3.5B images \cite{instagramnet}, and 1.8B image-text pairs \cite{align_google}. While such tremendously large-scale datasets are not always collectible, amplifying the dataset scale by synthesizing more data points by data augmentation techniques is common. Especially, \textit{mixed sample data augmentation} (MSDA) \cite{zhang2017mixup,tokozume2017learning,tokozume2018between,inoue2018data,yun2019cutmix,verma2019manifold,takahashi2019ricap,beckham2019adversarial,walawalkar2020attentive,yun2020videomix,chen2020gridmask,uddin2020saliencymix,kim2020puzzle,harris2020fmix,faramarzi2020patchup,qin2020resizemix,jeong2021observations,dabouei2021supermix,greenewald2021k,kim2021co,liu2021unveiling,chun2021styleaugment,jeong2021observations,sohn2022genlabel,liu2022decoupled} has become a standard technique to train a strong deep model by synthesizing a mixed sample from multiple (usually two) samples by combining both of their sample values and labels in a linear combination \cite{zhang2017mixup} or a cut-and-paste manner \cite{yun2019cutmix}. This simple idea, however, shows surprising performance enhancements in various applications, including image object recognition \cite{yun2019cutmix,kim2020puzzle,harris2020fmix,touvron2021training,dabouei2021supermix}, semi-supervised learning \cite{berthelot2019mixmatch, sohn2020fixmatch} self-supervised learning \cite{kalantidis2020hard,kim2020mixco,lee2021imix}, noisy label training \cite{Li2020DivideMix}, meta-learning \cite{yao2021improving}, semantic segmentation \cite{chang2020mixup, french2019semi}, natural language understanding \cite{guo2019augmenting, jindal2020leveraging}, and audio processing \cite{medennikov2018investigation,kim2021specmix,meng2021mixspeech}. Another advantage of MSDA beyond the performance improvements is that MSDA usually does not need domain-specific knowledge, such as strong image-specific \cite{randaug} or audio-specific \cite{specaug} transformations; hence MSDA can be universally employed by various applications. However, although MSDA shows excellent benefits in practice, there is still yet not enough understanding of how MSDA works well universally; \emph{can MSDA always show better generalization and robustness than the standard training with a theoretical guarantee?} Furthermore, the design choice of MSDA can be significantly varying and the optimal design choice is still ambiguous. For example, Lee \etal \cite{lee2021imix} showed that in self-supervised learning, Mixup is more effective than CutMix, while Ren \etal \cite{ren2022sdmp} observed opposite results.
The ambiguity is originated from the fact that we do not have a unified theoretical lens of understanding how different design choices affect the actual learning process; in short, \emph{how are Mixup and CutMix different?}

There have been several attempts to theoretically understand Mixup, a special case of MSDA \cite{zhang2020does,chidambaram2021towards,carratino2020mixup, zhang2021and}. They delve into the effect of Mixup in a loss function perspective, \eg, Mixup behaves as a regularization of the standard training \cite{zhang2020does}, or in a learning theory perspective, \eg, Mixup training can provide an upper bound for the true loss \cite{chidambaram2021towards,yao2021improving}. However, their analyses are limited to Mixup, while other MSDAs, such as CutMix, are poorly understandable through the lens of their analyses.
In this paper, we extend the theoretical results of Zhang \etal \cite{zhang2020does} and Chidambaram \etal \cite{chidambaram2021towards} to a general MSDA to provide a first unified theoretical lens for understanding how general MSDAs work by different choices of mixing strategies.
We show that MSDA behaves as an input gradient and Hessian regularization (\cref{thm::MSDA-loss}) as well as a regularizer for the first layer parameters; MSDA improves adversarial robustness (\cref{thm::MSDA-robustness}) and generalization (\cref{thm::MSDA-generalization}).
Our theoretical results show that popular MSDA methods, such as Mixup and CutMix, behave differently in terms of regularization effects.
Briefly, CutMix gives a strong regularization in the product of nearby distance pixel-level partial gradient and nearby distance Hessian of the estimated function $f$, while CutMix gives a weak regularization in the product of long-distance pixel-level partial gradient and long-distance Hessian of the estimated function $f$. In contrast, Mixup gives a regularization in gradient or Hessian of the estimated function $f$ regardless of the pixel-level distance.

From our unified theoretical lens for MSDA, we can conclude that \emph{there is no one-fit-all optimal MSDA fit to every data or model parameter}. In other words, the optimal mixing strategy depends on applications, datasets, and model architectures. It supports previous empirical observations that combining different MSDA methods (\eg, alternatively using Mixup and CutMix during training) can outperform using only one MSDA \cite{misra2019mish, bochkovskiy2020yolov4, touvron2021training, yun2021re}. From these observations, we propose two simple MSDA methods that naturally generalize Mixup and CutMix, so that it can take advantage of both methods. Our first proposed method, \oursfull \textbf{(\ours)}, mixes two samples in both Mixup and CutMix manners; it first cut-and-paste two samples as CutMix, and then it linearly interpolates the out-of-box values of two samples as Mixup. We let \ours be able to behave as both Mixup and CutMix by introducing a stochastic control parameter. Our second proposed method, \oursgaussianfull \textbf{(\oursgaussian)}, also mixes two samples in both Mixup and Cutmix manners; firstly we select a point, and then we mix two samples gradually using the Gaussian function. Our empirical results on CIFAR-100 and ImageNet show that \ours and \oursgaussian outperform the state-of-the-art MSDA methods, including Mixup, CutMix, and \stochmixupcutmix.
\section{A General Framework for Mixed Sample Data Augmentation (MSDA)}
\label{sec::prelim}
In this section, we define the formal definition of MSDA and notations. We define a training dataset as $D = \{ z_i = (x_i, y_i)\}_{i=1}^m$, randomly sampled from a distribution $\cP_{z}$. Here, $z = (x,y)$ is the input (\eg, an image) and output (\eg, the target class label) pair. Then, for randomly selected two data samples, $z_i$ and $z_j$, an augmented sample by MSDA, $\msda{z}{i,j}$, is synthesized as follows
\begin{align}
\label{eqn::msda}
\begin{split}
    \msda{z}{i,j}(\lambda, &1-\lambda) = (\msda{x}{i,j}(\lambda, 1-\lambda), \msda{y}{i,j}(\lambda, 1-\lambda))\\
    \text{where,} \quad &\msda{x}{i,j}(\lambda, 1-\lambda) = M(\lambda) \odot x_i + (1- M(\lambda)) \odot x_j \quad \text{and}\\
    &\msda{y}{i,j}(\lambda, 1-\lambda) = N(\lambda) \odot y_i + (1- N(\lambda)) \odot y_j,
\end{split}
\end{align}
where $\lambda$ is the ratio parameter between samples, drawn from $\cD_\lambda$ (usually Beta distribution).
$\odot$ means a component-wise multiplication in vector (or matrix). $M(\lambda)$ is a random variable conditioned on $\lambda$ that indicates how we mix the input (\eg, by linear interpolation \cite{zhang2017mixup} or by a pixel mask \cite{yun2019cutmix}). $N(\lambda)$ denotes a random variable conditioned on $\lambda$ that demonstrates how we combine the output. We assume that the output $y$ can be one-dimensional data or a matrix; the former means regression or classification task, and the latter means semantic segmentation task.
For the sake of simplicity, we let $y$ be one-dimensional data: $N(\lambda) = \lambda$ and $\bE[M(\lambda)] = \lambda\vec{1}$.

\paragraph{Remark 1.} If the meaning is not ambiguous, then we sometimes omit $\lambda$ (\ie, $M(\lambda)$ to $M$).
For the sake of simplicity, we consider mixing only two samples (\ie, $\msda{z}{i,j}(\lambda, 1-\lambda)$), but
we can similarly extend these analyses to mixing $n$-samples data augmentation \cite{takahashi2019ricap,jeong2021observations,greenewald2021k}. If we combine $n$-samples, the ratio parameter will be a vector in general
(See \cref{appendix::n-mixup}).

\paragraph{Remark 2.} As recent studies \cite{walawalkar2020attentive, kim2020puzzle, kim2021co} have shown, $M(\lambda)$ or $N(\lambda)$ can depend on $(z_i, z_j)$, \eg, by using a saliency map \cite{kim2020puzzle, kim2021co} or the class activation map \cite{walawalkar2020attentive}.
Since the proof techniques for our theoretical analysis are invariant to the choice of $M(\lambda)$ and $N(\lambda)$, our proof techniques also can be applied to the dynamic MSDA methods. For simplicity, we assume that $M$ is a random variable only depending on $\lambda$. In other words, we assume $\cW$ as a random sample space, and $M: \cW \times \Lambda \to \bR^n$ is a measurable function. We left the theoretical analysis of dynamic methods to the future.

Now, we re-write the two most popular MSDA methods, Mixup \cite{zhang2017mixup} and CutMix \cite{yun2019cutmix}, for $i$-th and $j$-th samples with $\lambda$, drawn from $\cD_\lambda$, by using the proposed framework (\cref{eqn::msda}) as follows:
\begin{align}
\label{eqn::msda-mixup}
\begin{split}
\text{\textbf{Mixup:}} \quad \mixup{z}{i,j}(&\lambda, 1-\lambda) ~=~ (\mixup{x}{i,j}(\lambda, 1-\lambda), \mixup{y}{i,j}(\lambda, 1-\lambda)) \qquad \qquad  \qquad \qquad\\
\text{where} \quad &\mixup{x}{i,j}(\lambda, 1-\lambda) ~=~ \lambda x_i + (1- \lambda) x_j \quad \text{and} \\
&\mixup{y}{i,j}(\lambda, 1-\lambda) ~=~ \lambda y_i + (1- \lambda) y_j.
\end{split}
\end{align}
\begin{align}
\label{eqn::msda-cutmix}
\begin{split}
\text{\textbf{CutMix:}} \quad &\cutmix{z}{i,j}(\lambda, 1-\lambda) ~=~ (\cutmix{x}{i,j}(M, 1-M), \cutmix{y}{i,j}(\lambda, 1-\lambda))\\
&\text{where} \quad \cutmix{x}{i,j}(\cutmix{M}{}, 1-\cutmix{M}{}) = \cutmix{M}{} \odot x_i + (1-\cutmix{M}{}) \odot x_j
\\ &\text{and} \qquad \cutmix{y}{i,j}(\lambda, 1-\lambda) = \lambda y_i + (1- \lambda) y_j.
\end{split}
\end{align}
Note that \cref{eqn::msda-mixup} is equivalent to \cref{eqn::msda} by putting $M(\lambda) = \lambda \vec{1}$.
In \cref{eqn::msda-cutmix}, $\cutmix{M}{}$ is a binary mask that indicates the location of the cropped box region with a relative area $\lambda$. 
Similarly, other MSDA variants can be easily formed as \cref{eqn::msda} by introducing new $M(\lambda)$ and $N(\lambda)$.

\paragraph{Notations.}
We define the loss function as $l(\theta, z)$, where $\theta \in \Theta \subseteq \bR^d$. We define $L(\theta) = \bE_{z \sim \cP_{z}} l(\theta,z)$ as the non-augmented population loss and $L_m(\theta) = \frac{1}{m}\sum_{i=1}^m l(\theta, z_i)$ as the empirical loss for the non-augmented population. For a general MSDA, we can define MSDA loss as
\begin{align}
    L^{\text{MSDA}}_m(\theta) &= \bE_{i, j \sim \text{Unif}([m])}{\bE_{\lambda \sim \cD_\lambda} \bE_M l(\theta, \msda{z}{i,j}(\lambda, 1-\lambda))}.\label{eqn::MSDA-original-loss}
\end{align}
Therefore, the Mixup and CutMix losses can be written as
\begin{align}
    L^{\text{mixup}}_m(\theta) & = \frac{1}{m^2}\sum_{i,j =1}^m\bE_{\lambda \sim \cD_\lambda} l(\theta, \mixup{z}{i,j}(\lambda, 1-\lambda)) \label{eqn::mixup-original-loss}
    \\
    L^{\text{cutmix}}_m(\theta) &= \frac{1}{m^2}\sum_{i,j =1}^m\bE_{\lambda \sim \cD_\lambda} \bE_M l(\theta, \cutmix{z}{i,j}(\lambda, 1-\lambda)), \nonumber
\end{align}
where $[m] = \{1,2,\dots, m\}$ and $\cD_\lambda$ is a distribution supported on $[0,1]$
with a conjugate prior.
Throughout this paper, we consider $\cD_\lambda$ as $\text{Beta}(\alpha, \beta)$, a common selection for $\lambda$ in practice. We define $\cD_X$ as the empirical distribution of the training dataset.

\section{A Unified Theoretical Understanding of MSDA}
\label{sec::approxloss}
In this section, we provide a unified theoretical lens of how MSDA works. Specifically, we follow the theoretical results for Mixup provided by Zhang \etal \cite{zhang2020does}, where Zhang \etal have shown that Mixup is equivalent to the summation of the original loss function and a Mixup-originated regularization term. We will give a general approximation form for MSDA using $\lambda \sim \text{Beta}(\alpha, \beta)$. We also show that our analysis can be extended to $n$-sample mixed augmentation 
(See \cref{appendix::n-mixup})

\paragraph{From an MSDA loss to an input gradient and Hessian regularization.}
We first consider the following class of loss functions for a twice differentiable prediction function $f_\theta(x)$ (\eg, a softmax output of a neural network), a twice differentiable function $h$, and target $y$:
$$\cL = \{ l(\theta, z) \,|\, l(\theta, z) = h(f_\theta(x)) - yf_\theta(x) \text{ for a twice differentiable function }h \}.$$ 
This function class $\cL$ includes the loss function induced by Generalized Linear Models (GLMs) and cross-entropy.
Now, we introduce our first theoretical result that induces an MSDA loss (\ie, \cref{eqn::MSDA-original-loss}) can be re-written as the summation of the original loss (the empirical loss for the non-augmented population loss, $L_m(\theta)$) and input gradient-related regularization terms as follows.

\begin{theorem}
\label{thm::MSDA-loss}
Consider a loss function $l \in \cL$. We define $\tilde{D}_\lambda$ as $\frac{\alpha}{\alpha + \beta}\text{Beta}(\alpha + 1, \beta) + \frac{\beta}{\alpha + \beta}\text{Beta}(\beta + 1, \alpha)$. Assume that $\bE_{r_x \sim \cD_X}[r_x] = 0$. Then, we can re-write the general MSDA loss \eqref{eqn::MSDA-original-loss} as
\begin{equation}
\label{eqn::thm1-approxloss}
L^{\text{MSDA}}_m(\theta) = L_m(\theta) + \sum_{i=1}^3 \cR_i^{(\text{MSDA})}(\theta) + \bE_{\lambda \sim \tilde{\cD}(\lambda)}\bE_M [(1-M)^\intercal \varphi(1-M) (1-M)], 
\end{equation}
where $\lim_{a \to 0}\varphi(a)= 0 $,
{\small
\begin{align}
\label{eqn::thm1-Rs}
\begin{split}
    &\cR^{(\text{MSDA})}_1(\theta) = \frac{1}{m} \sum_{i=1}^m( y_i - h'(f_\theta(x_i))) \left( \nabla f_\theta(x_i)^\intercal x_i\right) \bE_{\lambda \sim \tilde{D}_\lambda} (1-\lambda),
    \\
    &\cR^{(\text{MSDA})}_2(\theta) = \frac{1}{2m}\sum_{i=1}^m h''(f_\theta(x_i)) \bE_{\lambda \sim \tilde{D}_\lambda} \fG(\cD_X, x_i, f, M)  ,  
    \\
    &\cR^{(\text{MSDA})}_3(\theta) = \frac{1}{2m}\sum_{i=1}^m(h'(f_\theta(x_i))-y_i) \bE_{\lambda \sim \tilde{D}_\lambda} \fH(\cD_X, x_i, f, M),
\end{split}
\end{align}
and 
\begin{align}
\label{eqn::thm1-GH}
\begin{split}
    \fG(\cD_X, &x_i, f, M) = \bE_M (1-M)^\intercal \bE_{r_x \sim \cD_X} \left(\nabla f(x_i) \odot (r_x- x_i) \left(\nabla f(x_i) \odot (r_x- x_i)\right)^\intercal\right)(1-M)
    \\
    &= \sum_{j,k \in \text{coord}}a_{jk} \partial_j f_\theta(x_i)\partial_k f_\theta(x_i)\left(\bE_{r_x \sim \cD_X}[r_{xj}r_{xk}] + x_{ij}x_{ik}\right),
    \\
    \fH(\cD_X, &x_i, f, M) = \bE_{r_x \sim \cD_X}\bE_M (1-M)^\intercal  \left(\nabla^2f_\theta(x_i) \odot \left((r_x- x_i)(r_x- x_i)^\intercal\right)\right)(1-M)
    \\
    &= \sum_{j,k \in \text{coord}} a_{jk} \left(\bE_{r_x \sim \cD_X}[r_{xj}r_{xk}\partial^2_{jk}f_\theta(x_i)] + x_{ij}x_{ik}\partial^2_{jk}f_\theta(x_i)\right),
\end{split}
\end{align}
where

\begin{equation}
\label{eqn::regularization_coefficient}
    a_{jk} := \bE_M[(1-M_j)(1-M_k)].
\end{equation}
}
\end{theorem}
\begin{proof}[Proof outline of Theorem~\ref{thm::MSDA-loss}]
Using the definition of $\tilde{z}_{ij}$ and using the fact that the Binomial distribution and Beta distribution are in the conjugate, we can reformulate $L_m^{(MSDA)}$. \revision{In the process of reformulating $L_m^{(MSDA)}$, we should define $\tilde{\cD}_\lambda$.} Then, we can make a quadratic Taylor approximation of the loss term. \revision{Here, $\bE_{r_x}[r_x]=0$ is used for not only the simplicity of the results, but also for the fact that using normalization in the dataset.} Details can be found in
\cref{appendix::thm1pf}.
\revision{We also show that \cref{thm::MSDA-loss} can be extended to $n$-sample MSDA methods (\cref{appendix::n-mixup}). In this case, the combinatorial terms in quadratic multivariate Taylor approximation also come out.}
\end{proof}
\revision{
\textit{How is our approximation accurate?}
We call $\tilde{L}_m^{MSDA}(\theta):= L_m(\theta) + \sum_{i=1}^3 \cR_i^{(MSDA)}$ as \textit{the approximate MSDA loss}.
Here, we empirically demonstrate that our quadratic approximation is almost accurate by following numerical validations in \cite{wager2013dropout, carratino2020mixup, zhang2020does}.
Specifically, we train logistic regression models on two-moons dataset \cite{buitinck2013api} in two ways: (1) by using the original MSDA loss function (2) by using our approximated loss function. We employ two MSDA examples as below.
\begin{itemize}
    \item The original Mixup \ie, $\lambda \sim \text{Beta}(1,1)$ and $M = \lambda \vec{1}$
    \item Variants of CutMix \ie, $\lambda \sim \text{Beta}(1,1)$ and $M = (m_1, m_2)$ such that $m_i \sim \text{Bernoulli}(\lambda)$. 
\end{itemize} 
\cref{fig:approx} displays the approximate loss function and the original loss function. According to empirical findings, we can conclude that the original MSDA loss is fairly close to the approximate MSDA loss.
}
\begin{figure}
\centering
\begin{subfigure}{0.45\textwidth}
    \includegraphics[width=\textwidth]{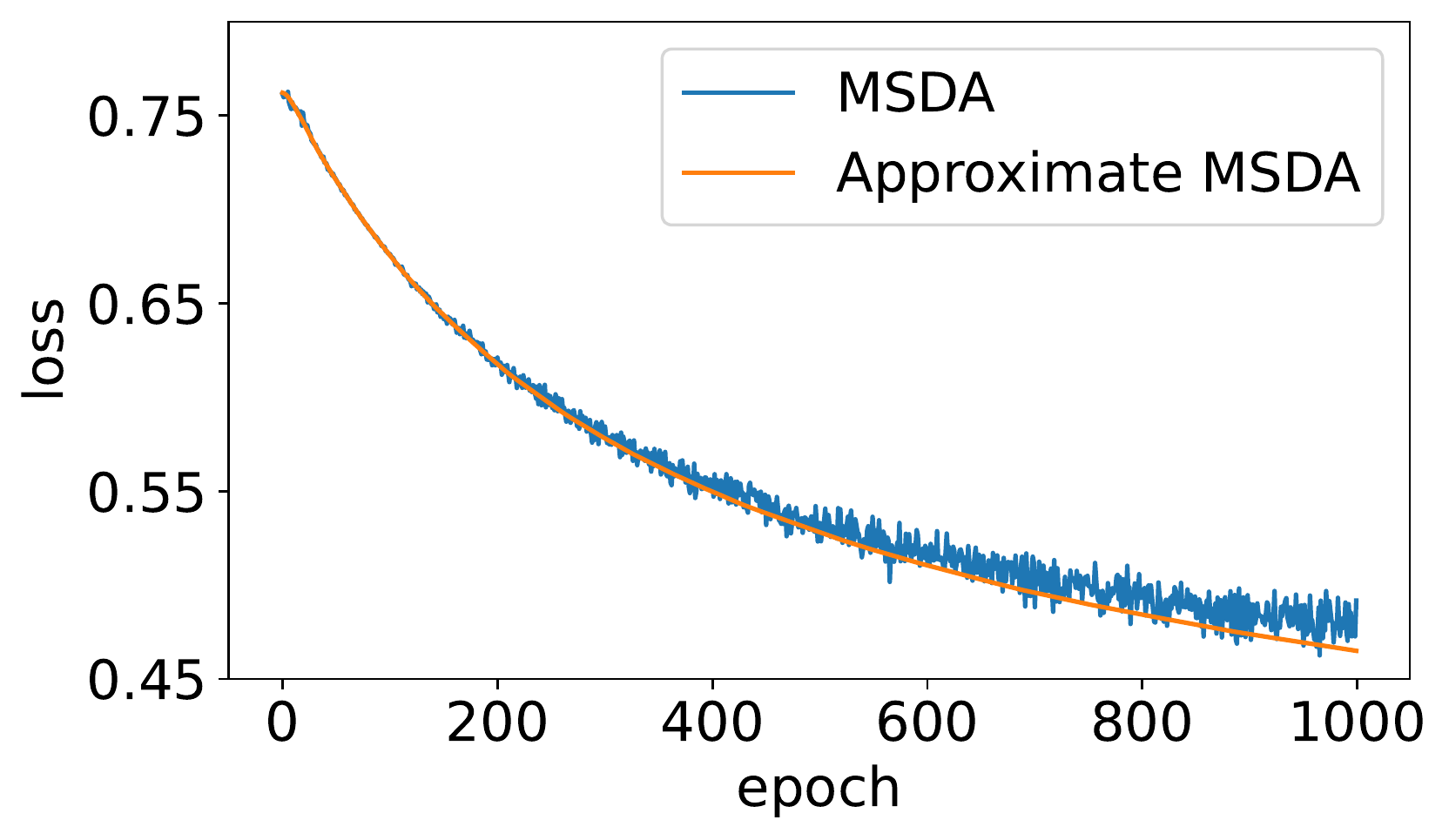}
    \caption{Mixup}
\end{subfigure}
\begin{subfigure}{0.45\textwidth}
    \includegraphics[width=\textwidth]{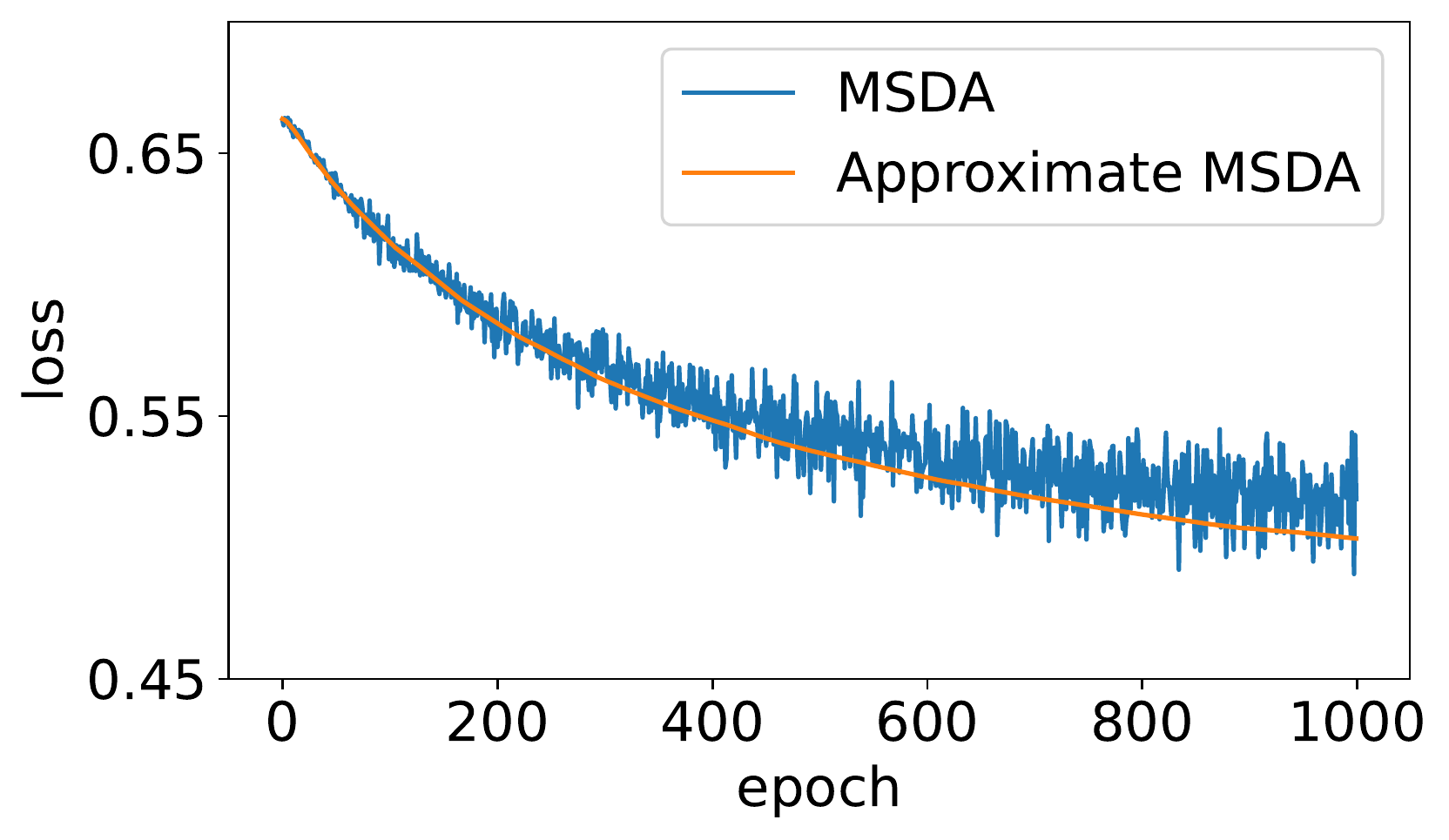}
    \caption{CutMix}
\end{subfigure}
\caption{Comparison of the original MSDA loss with the approximate MSDA loss function.}
\label{fig:approx}
\end{figure}

\textit{What makes the difference between various MSDA methods?}
In the theorem, as we define $E_M(1-M) = 1-\lambda$, $\cR_1^{(\text{MSDA})}$ is the same for every MSDA method. Namely, the difference between MSDA methods originated from $\cR_2^{(\text{MSDA})}$ and $\cR_3^{(\text{MSDA})}$.
Note that if we set $M$ = $\lambda \vec{1}$, \cref{thm::MSDA-loss} indicates a Mixup loss \eqref{eqn::mixup-original-loss}, and the result is consistent with Zhang \etal \cite{zhang2020does}.
In \cref{eqn::thm1-Rs} and \cref{eqn::thm1-GH}, we observe that $\cR_2$ is related to the input gradient $\nabla f_\theta(x_i)$ and $\cR_3$ is related to input Hessian $\nabla^2 f_\theta(x_i)$ with mask-dependent coefficients $a_{jk}$ \eqref{eqn::regularization_coefficient}.
In other words, different design choice of MSDA (\eg, how to design $M$) will lead to different magnitudes of regularization on the input gradients and Hessians.
Because the values of input gradients and Hessians are varying by datasets, tasks and model architecture choices, we can conclude that the optimal choice of $M$ is dependent on the applications.
We also describe how the other MSDA methods (\eg, dynamic MSDAs \cite{uddin2020saliencymix, kim2020puzzle, kim2021co}) can be interpreted through the lens of our unified analysis in \cref{appendix::other-msdas}.

In addition, to show that MSDA behaves as a regularization on input gradients and Hessians for any desired $a_{jk}$, we also show that there always exists a MSDA design choice $M$ for any desired regularization coefficient matrix $A(\lambda):= (a_{jk}(\lambda))$ with the regular conditions.
\begin{theorem}
\label{thm::mask}
For the given $\lambda$, we assume $A(\lambda) - (1-\lambda)^2\vec{1}{\vec{1}^\intercal}$ is a nonnegative definite matrix. Then we can construct a real-valued mask $M$ that $\bE(1-M_j)(1-M_k) = a_{jk}$ for all $j,k$.
\end{theorem}
\begin{proof}
Setting $M = 1-\lambda + (A(\lambda) - (1-\lambda)^2\vec{1}{\vec{1}^\intercal})^{1/2} Z$ where $Z$ is normal distribution, the theorem holds.
\end{proof}
Note that, in the proof, $M$ values are not bounded where typically we choose $0\leq M_i \leq 1$.
In other words, the theorem holds if we allow mask values out of $[0, 1]$.
To investigate the potentiality of unbounded mask, we explore Mixup with unbounded masks in Figure~\ref{fig:negative-M}. Although, allowing negative values to $M$ can be beneficial, we leave a new mask design with unbounded values as a future work.

\begin{figure}[t]
\centering
\includegraphics[height=2cm]{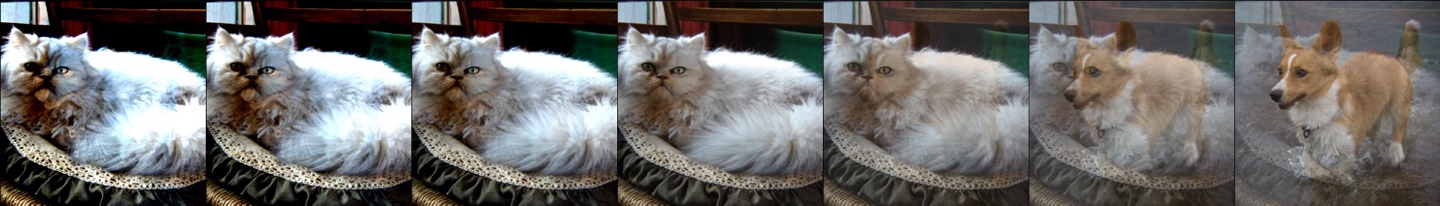}
\caption{The negative $M$ value results. the image is $\lambda * \text{dog} + (1-\lambda) * \text{cat}$ where $-0.75 \leq \lambda \leq 0.75$}
\label{fig:negative-M}
\vspace{-1em}
\end{figure}

Unfortunately, as the target loss function \eqref{eqn::thm1-approxloss} is mingled with the choice of mask $M$, data sample $x_i$, and pixel-level function gradient, the optimal choice of mixing strategy $M$ is not achievable in the closed-form solution. Instead, \cref{thm::MSDA-loss} implies that there is no absolute superiority between the design choice of MSDA, but it depends on datasets and the target tasks, as our empirical observation is consistent with the theoretical interpretation.
In \cref{sec::meanofmask}, we will provide more examples of how different $M$ affects the actual coefficients $a_{jk}$ and the input gradients for better understanding.

Using the regularization term $\cR_2^{(\text{MSDA})}$ \eqref{eqn::thm1-Rs}, we can also provide a theoretical connection between MSDA methods and the notion of flatness where a more flat solution leads to better generalization in applications \cite{keskar2016largebatch,garipov2018fge,izmailov2018swa,foret2020sharpness,cha2021swad}.
Inspired by Ma \etal \cite{ma2021linear}, we split the parameters by $\theta= (\theta_1, \theta_2)$, and then the neural network can be represented by the form $f_\theta(x) = \tilde{f}_{\theta_2}(\theta_1 x)$. Therefore, we have
\begin{align*}
    \nabla_{\theta_1} \tilde{f}_{\theta_2}(\theta_1 x) = \fracpartial{f}{(\theta_1 x)} x^\intercal,  \qquad \nabla_{x} \tilde{f}_{\theta_2}(\theta_1 x) = \theta_1^\intercal \fracpartial{f}{(\theta_1 x)},
\end{align*}
where $\fracpartial{f}{(\theta_1 x)}$ is the partial derivative of the first layer. Now, we have
\begin{align*}
\label{eqn::parameter-smootheness}
\begin{split}
    ((1-M)\odot \nabla f(x))^\intercal x &= \text{tr}(x ((1-M)\odot \nabla f(x))^\intercal) = \text{tr}\left(x \left(\theta_1^\intercal \fracpartial{f}{\theta_1 x} \odot (1-M)\right)^\intercal\right)
    \\
    &= \text{tr}\left(x \left(\left( \fracpartial{f}{\theta_1 x}\right)^\intercal \theta_1 \diag(1-M) \right)\right) 
    \\
    &= \text{tr}\left(\left(\nabla_{\theta_1}\tilde{f}_{\theta_2}(\theta_1 x)\right)^\intercal \theta_1 \diag(1-M)\right).
\end{split}
\end{align*}
Note that the terms in $\fG$ \eqref{eqn::thm1-GH} can be re-written as follows
\begin{align*}
    \sum_{j,k \in \text{coord}} \bE_M[(1-M_j)(1-M_k)] \partial_j f_\theta(x_i)\partial_k f_\theta(x_i) (x_{ij} x_{ik}) = \bE \left[ (((1-M)\odot \nabla f(x))^\intercal x)^2 \right].
\end{align*}
In other words, by minimizing the regularization term $\cR_2^{(\text{MSDA})}$, $\int (\theta_1^\intercal \nabla_{\theta_1} \tilde{f}_{\theta_2})^2$, \ie, the regularization effect of flatness at the interpolation solution can be minimized in a sample-wise weighted manner.
Therefore, the regularization term $\cR_2^{(\text{MSDA})}$ also can be interpreted as a regularization of the first layer parameters and their partial derivative of $f$.

\paragraph{Robustness and generalization properties of MSDA.}
As a number of studies \cite{moosavi2019robustness,mustafa2020input, ma2021linear} have shown that regularizing input gradient and Hessian will give better robustness and generalization to the target network $\theta$, it can be shown that MSDA also has adversarial robustness properties and generalization properties based on Theorem~\ref{thm::MSDA-loss}.
The full statement of Theorem~\ref{thm::MSDA-robustness} and Theorem~\ref{thm::MSDA-generalization} can be found in 
\cref{appendix::thm2pf} and \cref{appendix::thm3pf}, respectively.

\begin{theorem}[Informal]
\label{thm::MSDA-robustness}
With the logistic loss function under the ReLU network, the approximate loss function of MSDA is greater than the adversarial loss with the $\ell_2$ attack of size $\epsilon\sqrt{d}$. 
\end{theorem}
\revision{\begin{proof}[Proof outline of \cref{thm::MSDA-robustness}]
Defining adversarial loss function and using second order taylor expansion, we can prove that adversarial loss is less than MSDA loss.
\end{proof}}
\begin{theorem}[Informal]
\label{thm::MSDA-generalization}
Under the GLM model and the regular conditions, and if we use MSDA in training, we have
\begin{align*}
    L(\theta) \leq \tilde{L}_m^{(\text{MSDA})}(\theta) + \sqrt{\frac{\cO\left(\log(1/\delta)\right)}{n}}
\end{align*}
with probability at least $1-\delta$. This also holds for the MSE loss and a feature-level MSDA. 
\end{theorem}
\revision{
\begin{proof}[Proof outline of \cref{thm::MSDA-generalization}]
MSDA regularization can be altered to the original empirical risk minimization problem with a constrained function set, and calculating Radamacher complexity of this function set gives the theorem.
\end{proof}}
In addition to \cref{thm::MSDA-robustness} and \cref{thm::MSDA-generalization}, we can prove that the optimal solution of \eqref{eqn::MSDA-original-loss} can achieve a perfect classifier (\ie, classifies every augmented sample $x$ correctly) in the logistic classification setting by following Chidambaram \etal \cite{chidambaram2021towards}. The full statement are in
\cref{appendix:chidambaram}.

\paragraph{Summary.}
Our unified theoretical lens for MSDA shows that for any MSDA method formed as \cref{eqn::MSDA-original-loss}, the method satisfies that (1) it behaves as a regularizer of input gradients, Hessian, and the first layer parameters (\cref{thm::MSDA-loss}); (2) there exists a mask $M$ for any desired regularization coefficients $a_{jk}$ (\cref{thm::mask}) (3) it achieves better adversarial robustness (\cref{thm::MSDA-robustness}) and generalization (\cref{thm::MSDA-generalization}) than the vanilla training.
Interestingly, \cref{thm::MSDA-loss} shows the difference between different MSDA design choices (\eg, different $M$, such as linear interpolation \cite{zhang2017mixup}, cropped box \cite{yun2019cutmix}) will lead to different magnitudes of the input gradient regularization \eqref{eqn::thm1-Rs}.

\section{Comparison of Different MSDA Design Choices: The Role of Masks}
\label{sec::meanofmask}

As we observed in the previous section, different design choices for MSDA (\ie, the choice of $M$) affect to the degree of the regularization in \cref{thm::MSDA-loss} (\ie, $\cR_2^{(\text{MSDA})}$ and $\cR_3^{(\text{MSDA})}$) by the relationships of pixels. In this section, we show how different MSDA methods lead to different regularization effects by empirical studies; we first show the values of the regularization coefficients $a_{jk}$ by varying masks; then we show the input gradient values that are regularized by $a_{jk}$ \eqref{eqn::regularization_coefficient} after the MSDA training; finally, we show that the best choice of the mask design can be varying by the target task settings.
In addition, we propose two generalized versions of Mixup and CutMix, called \ours and \oursgaussian, that empirically show the intermediate property of Mixup and CutMix.

\paragraph{Introduction to \ours and \oursgaussian.}
Recall that the regularization coefficients $a_{jk}$ is determined by $M$ (\cref{eqn::regularization_coefficient}). For example, by choosing $M$ = $\lambda \vec{1}$ (\ie, Mixup), $a_{jk}$ is always $(1-\lambda)^2$.
On the other hand, the result slightly changes for CutMix: $a_{jk}$ depends on how $j$ and $k$ are close. Informally, due to dependency between $M_j$ and $M_k$ (as $M$'s component is always 0 in the cropped box regions and 1 in others), close $j$ and $k$ give large $a_{jk}$, but distant $j$ and $k$ give small $a_{jk}$. 
$ a_{jk}$ is calculated as
{\small
\begin{equation}
\label{eqn::cutmix-coef}
    a_{jk} =\frac{\max(\min(h(j_1) - l(k_1), h(k_1) - l(j_1)), 0)  \max(\min(h(j_2) - l(k_2), h(k_2) - l(j_2)), 0)}{(n - [\sqrt{1-\lambda}n])^2}    
\end{equation}
}
where $j = (j_1, j_2), k= (k_1, k_2), h(t) = \min(t, n-[\sqrt{1-\lambda}n]), l(t) = \max(t- [\sqrt{1-\lambda}n], 0)$.

We visualize $a_{jk}$ of different MSDA methods in \cref{fig:coefcomp}. We compare Mixup, CutMix, \stochmixupcutmix. We also propose two generalized MSDA methods, named \textbf{\ours} and \textbf{\oursgaussian}. 
Before comparing the methods, we first formally define \stochmixupcutmix, \ours and \oursgaussian. These methods can be formed as \eqref{eqn::msda}
where the definition of $M$ is varying by the methods.

\textit{\stochmixupcutmix} is a practical variant of MSDA by considering Mixup and CutMix at the same time. By a simple alternation of two augmentations, the state-of-the-art performances on large-scale datasets are shown \cite{touvron2021training,wightman2021resnet}.
\stochmixupcutmix is the same as \cref{eqn::msda} by setting $M(\lambda) = (1-\lambda) \vec{1}$ with probability $q$ and $M(\lambda) = M^{\text{cutmix}}(\lambda)$ with probability $1-q$. We choose $q=0.5$ as \cite{touvron2021training,wightman2021resnet}. In our loss function perspective, the regularizing coefficient terms (\ie, $\cR_2, \cR_3$) become the average of Mixup and CutMix's regularization coefficient. Namely, let $a_{ij}^\text{mixup} = (1-\lambda)^2$ be a regularization coefficient of Mixup and $a_{ij}^\text{cutmix}$ be a regularization coefficient of CutMix \eqref{eqn::cutmix-coef}, then the regularization coefficients of \stochmixupcutmix is
$q a_{ij}^\text{cutmix} + (1-q)a_{ij}^\text{mixup}$.

Here, we additionally propose two MSDA variants, \ours and \oursgaussian, that leverage the advantages of Mixup and CutMix, resulting in showing the intermediate property between Mixup and CutMix.

\begin{figure}[t]
    \centering
    \includegraphics[width=.95\linewidth]{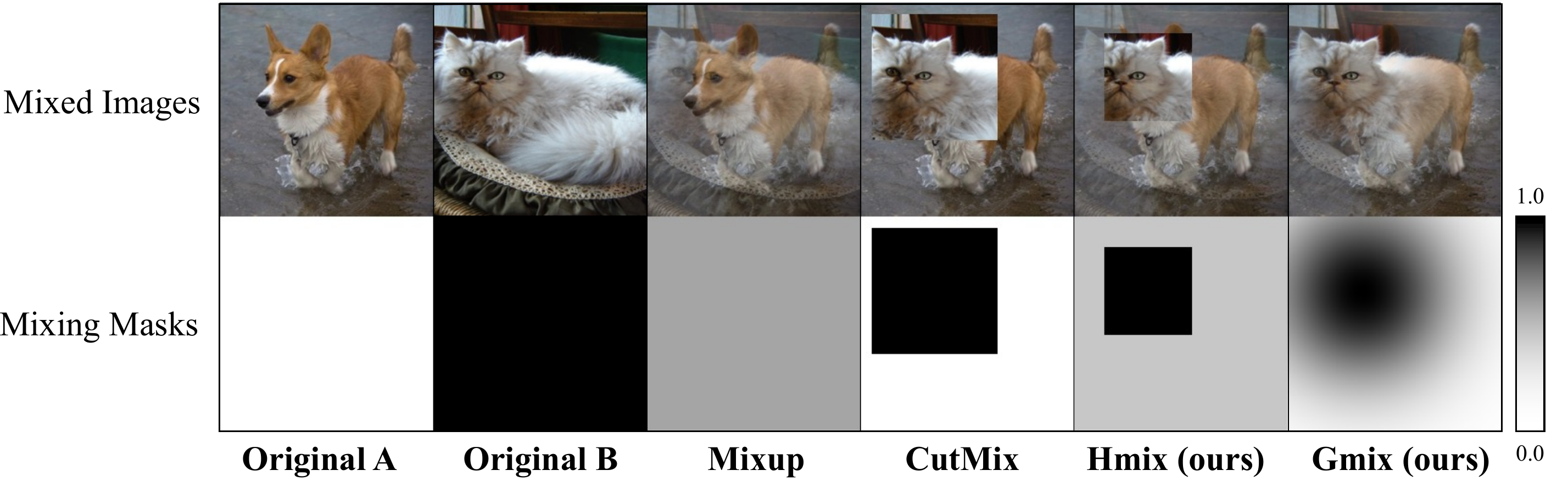}
    \caption{\small {\bf Examples generated by different MSDAs.}  From left to right, two original images to be mixed, Mixup, CutMix sample, \ours, and \oursgaussian.
    The first and the second rows show generated samples and their mixing masks $M$, respectively.
    We set $\lambda=0.65$ for all images and $r=0.5$ for \ours.
    }
    \label{fig:msda_samples}
    \vspace{-.5em}
\end{figure}
\textit{\oursfull (\ours)} combines Mixup and CutMix by shrinking the CutMix cropped box region and linearly interpolating two images in the areas out of the box as Mixup. The shrinking ratio of the cropped box region is determined by the ratio $r$. \ours can be written as \eqref{eqn::msda} by setting $M$ by (1) randomly cropped box region with side length $\sqrt{1-\lambda}\sqrt{r}N$ where $N$ is the side length of the original image, and make $M$'s component in the box region as 0 (2) in the areas other than the box, we set $M_i$ as $\frac{\lambda}{1 - (1-\lambda)r}$. We can easily check that $\bE[M] = \lambda \vec{1}$. As $r \to 0$, this method goes to Mixup, and as $r \to 1$, this method goes to CutMix.
Note that the ratio $r$ can be a random variable, such as $Beta(\gamma, \gamma)$.
In this case, if we set $\gamma \to 0$, as $Beta(\gamma, \gamma)$ goes to Bernoulli distribution, this is equivalent to \stochmixupcutmix. 

We propose \textit{\oursgaussianfull (\oursgaussian)} to relax the CutMix box condition to a continuous version as the rectangle cropping of CutMix causes implausible augmented data, \eg, the boundary between two mixed samples. Therefore, we combine two ideas of Mixup and CutMix. Firstly, we select a point $p$ from the given input. Then, we make $M_i$ as the related function with $\norm{i-p}^2$. Specifically, we use the Gaussian function for making $M$: (1) randomly select a point $p$ in image (2) in the areas other than the box, we set $M_i$ as $1 - \exp\left(-\frac{\norm{i-p}^2 \pi}{2(1 - \lambda) N^2}\right)$.
The proposed \oursgaussian has the following $a_{ij}$
\begin{align}
    a_{ij} & = \frac{1}{N^2}\sum_{p \in \text{pixel}} \exp\left(\frac{-\pi}{2(1 - \lambda) N^2}\left(-\norm{i-p}^2 - \norm{j - p}^2\right)\right) \nonumber \\
    &~=\int_{\bR^2} \exp\left(\frac{-\pi}{2(1 - \lambda) N^2}\left(-\norm{i-p}^2 - \norm{j - p}^2\right)\right) dx  \nonumber
    \\
    &= (1-\lambda) \exp\left(\frac{-\pi}{(1 - \lambda) N^2}\norm{\frac{i-j}{2}}^2\right). \label{eqn::gaussian-mixup}
\end{align}
As seen in \cref{eqn::gaussian-mixup}, $a_{ij}$ smoothly goes down when the pixel distance becomes larger.

\cref{fig:msda_samples} shows the examples generated Mixup, CutMix, \ours, and \oursgaussian. 
The proposed methods (\ours and \oursgaussian) generate images in a hybrid form with the properties of both Mixup and CutMix.

\paragraph{Comparison in terms of regularization coefficients $a_{jk}$.}
We illustrate the regularization coefficients $a_{jk}$ of the different MSDA methods in \cref{fig:coefcomp}. In particular, we fix the mask parameter $\lambda$ to 0.5 and the input resolution to 64 $\times$ 64. \cref{fig:coefcomp} shows the difference between the MSDA methods in terms of how they regularize the input gradients and input Hessians: Mixup has equal weights to every gradient component or Hessian component, while CutMix gives high regularization in close coordinate gradient products or Hessian.
We also observe that the hybrid methods (\eg, \stochmixupcutmix, \ours, and \oursgaussian) show the intermediate coefficient values of Mixup and CutMix.

\begin{figure}[t]
    \centering
    \includegraphics[width=.95\linewidth]{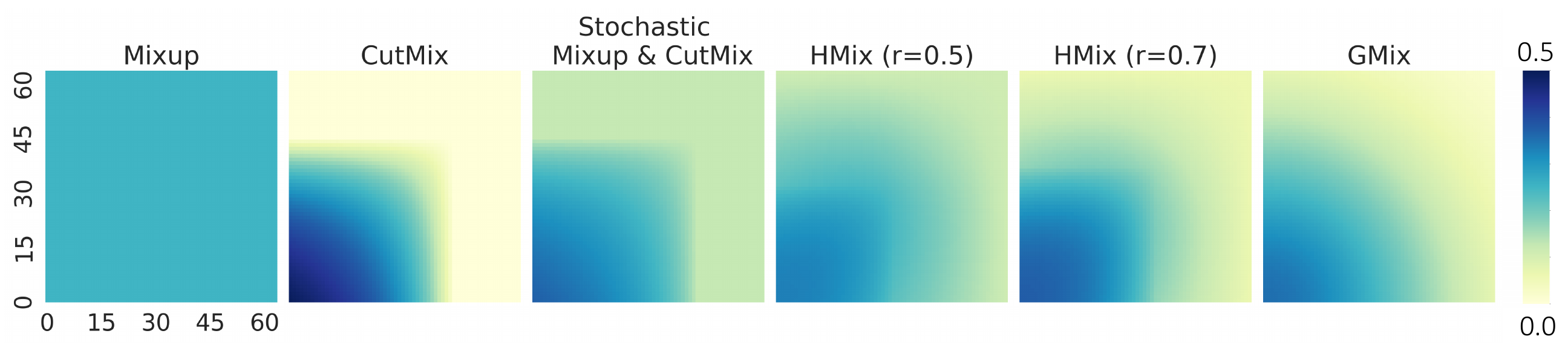}
    \caption{\small {\bf Visualization of regularization coefficients for different MSDA methods.} $a_{ij}$ values of Mixup, CutMix, \stochmixupcutmix (the alternation of Mixup and CutMix), \ours, \oursgaussian (described in \cref{sec::meanofmask} and \cref{appendix::hmix-regularizer}) are shown.
    Each $(x,y)$ value is computed by $\bE_i a_{i,i+(x,y)}$ where $i$ is a pixel vector.
    }
    \label{fig:coefcomp}
\end{figure}

\revision{
\paragraph{Comparison in terms of the regularized input gradients after MSDA training.}
Equation \eqref{eqn::thm1-GH} shows that the regularization term $a_{ij}$ directly affects to the pixel gradients $|\partial_i f_\theta(x_k) \partial_{j} f_\theta(x_k)|$ in our approximated loss function. The purpose of \cref{fig:mixup-cutmix-comp} is to show how the pixel gradients are actually regularized after training. We investigate the amount of the regularized input gradients by $|\partial_v f_\theta(x) \partial_{v+p} f_\theta(x)|$ with respect to the pixel distance vector $p$ for trained models by different MSDA methods. Here, if our approximated loss function actually behaves as a regularization, then we can expect that the pixel gradients $|\partial_v f_\theta(x) \partial_{v+p} f_\theta(x)|$ is small when $a_{ij}$ is large for the given $p$.
 
We first define the partial gradient product as follows: 
\begin{align}
    \text{PartialGradProd}(x,p) = \max_{v} |\partial_v f_\theta (x) \partial_{v+p} f_\theta (x)|\label{eqn::partialgradprod}
\end{align}
Now, we visualize the pixel-wise maximum values of $\text{PartialGradProd}(x,p)$ in \cref{fig:mixup-cutmix-comp}. We train different models $f_\theta$ on resized ImageNet (64 x 64) and measure the values on the validation dataset. The $x$-axis and $y$-axis of \cref{fig:mixup-cutmix-comp} denote the pixel distance $p$ along each $x$ and $y$ axis, and the scale of the colorbar denotes the value of the maximum partial gradient product. In the figure, we can observe that CutMix reasonably regularizes effectively in the input gradients products when a pixel distance is small; these results aligned with our previous interpretation, CutMix behaves a pixel-level regularizer where it gives stronger regularization (larger $a_{ij}$) to the closer pixels.
Note that we are not discussing the relationship between regularizing effects and accuracy but discussing regularizing coefficients and the optimized function's pixel gradients. }

\begin{figure}
\centering
\begin{subfigure}{0.28\textwidth}
    \includegraphics[width=\textwidth]{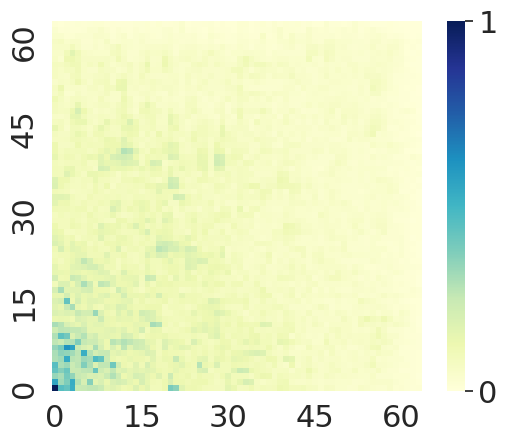}
    \caption{Vanilla (no MSDA)}
    \label{fig:first}
\end{subfigure}
\hfill
\begin{subfigure}{0.28\textwidth}
    \includegraphics[width=\textwidth]{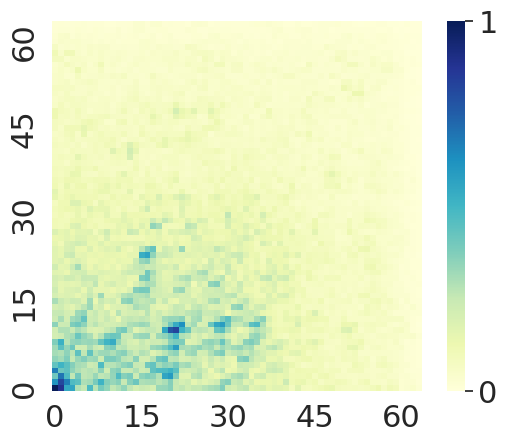}
    \caption{Mixup}
    \label{fig:second}
\end{subfigure}
\hfill
\begin{subfigure}{0.28\textwidth}
    \includegraphics[width=\textwidth]{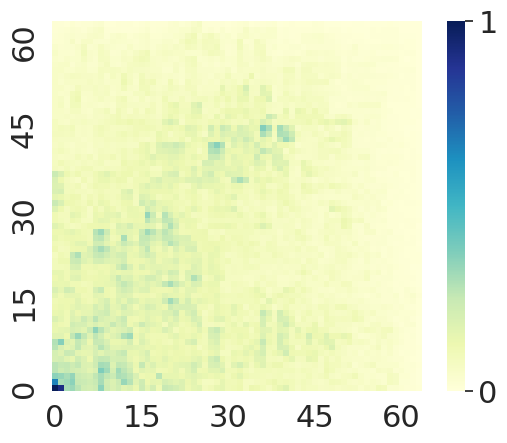}
    \caption{CutMix}
    \label{fig:third}
\end{subfigure}
    
\caption{\small {\bf Regularized input gradients by MSDA.} The normalized pixel-wise partial gradient norm product comparison among the models trained with vanilla setting (a), Mixup (b) and CutMix (c). \revision{We plotted \eqref{eqn::partialgradprod}, and $x$ and $y$ axis denote the pixel distance $p$ along each axis.}
} 
\label{fig:mixup-cutmix-comp}
\vspace{-1em}
\end{figure}

\paragraph{Understanding application cases when a specific MSDA design choice works better than others.}
From our theoretical results and empirical studies, we have shown that the design choice of MSDAs (\ie, $M$) leads to different regularization effects by regularization coefficient $a_{jk}$. Furthermore, as we have shown in \cref{thm::mask}, there always exists a mask that can form any desired $a_{jk}$.
We hypothesize that for the given dataset, if a short distance relation is relatively more important than longer distance relations, then CutMix will be better than Mixup. On the contrary, in the opposite case, if a short distance relation is relatively less important, then Mixup will be better than CutMix. 

Here, we study different task scenarios when different $a_{jk}$s are required by controlling the pixel-level importance of ImageNet-100 \cite{tian2019contrastive} training images.
In particular, we design two different scenarios where each of them needs different regularization strategies due to the different pixel-level importance of each task.
The results are shown in \cref{table:ablation}. 
\revision{We also report the performance of our proposed methods in both scenarios 1 and 2 in \cref{appendix::additional-experiments}.}

\paragraph{Scenario 1: Smaller objects by large crop size.} 
We randomly crop a large region (80\% to 100\%) of an image and resize to $64 \times 64$ to train a model. 
As the objects in the image become small, a close-distance relationship might be more important than a large-distance relationship. Here, we expect CutMix performs better than Mixup as shown in \cref{table:ablation}.

\paragraph{Scenario 2: Larger objects by small crop size} 
We randomly crop a small region (25\% to 40\%) of an image and resize to $64 \times 64$ to train a model. 
Contrary to Scenario 1, the objects in the image would become large in the cropping region and the large-distance relationship might be important, 
therefore, we expect that Mixup performs better than CutMix. This hypothesis is aligned to \cref{table:ablation}.

\begin{table}[t]
\centering
\small
\caption{\small \textbf{Different tasks need different MSDA strategies.} Validation accuracies of Mixup and CutMix trained networks on two different scenarios on ImageNet-100. Each scenario assumes different pixel importances.}
\label{table:ablation}
\vspace{.5em}
\begin{tabular}{lccc}
\toprule
                                   & {Mixup}       & {CutMix}     & $\Delta$ (CutMix - Mixup)  \\ \midrule
{Scenario 1: Large crop}                       &   58.3                    &    \textbf{64.4}      & \textcolor{darkergreen}{\textbf{+6.1}}             \\ 
{Scenario 2: Small crop}                       &   \textbf{67.7}                    &     67.0     & \textcolor{red2}{\textbf{-0.7}}            \\ 
\bottomrule
\end{tabular}
\vspace{-1em}
\end{table}

\section{Comparison of Different MSDA Design Choices: An Empirical Validation}
\label{sec::experiment}

In this section, we compare various MSDA methods on two popular large-scale image classification benchmarks: CIFAR-100 \cite{krizhevsky2009learning} and ImageNet-1K \cite{deng2009imagenet}.
We will confirm that our proposed design choices, \ours and \oursgaussian, are not only theoretically interpolating Mixup and CutMix in the toy settings, but also taking benefits of each method by showing great performances in real-world applications.
\revision{The implementation details and the hyper-parameter study can be found in
\cref{appendix::experimental}.}

\paragraph{Results on CIFAR-100 classification.}
We evaluate our method (\ours and \oursgaussian) against baseline MSDA methods including Mixup \cite{zhang2017mixup}, CutMix \cite{yun2019cutmix}, \stochmixupcutmix, \cite{touvron2021training} and PuzzleMix \cite{kim2020puzzle} on CIFAR-100 dataset.
Here, we include PuzzleMix, to see the effectiveness of our data-agnostic method against the data-aware mask strategy.
Note that although our theoretical results (\cref{sec::approxloss}) are based on the data-agnostic mask selection methods, our theoretical results can be easily extended to the data-dependent mask selection methods. We leave the extension as a future research direction.

To see the generalizability of our methods, we train various network architectures including ResNet-56 (RN56) \cite{he2016deep}, WideResNet28-2 (WRN28-2) \cite{zagoruyko2016wide}, PreActResNet-18 (PreActRN18) \cite{he2016identity}, PreActResNet-34 (PreActRN34) \cite{he2016identity} and PreActResNet-50 (PreActRN50) \cite{he2016identity} with various MSDA methods. 
We train networks for 300 epochs using SGD optimizer with  a learning rate $0.2$.
\cref{table:cifar-classification} shows the summarized results. 
We set the hyper-parameter $\alpha$ for Mixup, CutMix, and \stochmixupcutmix to $1$.
$\alpha$ for \ours and \oursgaussian were set to $1$ and $0.5$, respectively. 
We use $r=0.5$ for \ours.
In the table, \ours and \oursgaussian outperform Mixup only and CutMix only counterparts and \stochmixupcutmix often show comparable performances to \ours and \oursgaussian. 
Our methods show comparable performance with the state-of-the-art data-dependent strategy PuzzleMix.

\begin{table}[t]
\centering
\small
\caption{\small \textbf{CIFAR-100 classification.} Comparison of various MSDA methods on various network architectures. Note that \textcolor{gray}{PuzzleMix} needs additional computations (twice than others) for computing the input saliency.}
\label{table:cifar-classification}
\vspace{.5em}
\begin{tabular}{lccccc}
\toprule
Augmentation Method                  & RN56 & WRN28-2 &  PreActRN18 &  PreActRN34 &  PreActRN50 \\ \midrule
Vanilla (no MDSA)   & 73.23 & 73.50    & 76.73 & 77.68 &	79.07 \\
Mixup               & 73.12 & 74.05    & 77.21 & 79.02 &	79.34 \\
CutMix              & 74.83 & 74.79    & 78.66 & 80.05 &	81.23 \\
\textcolor{gray}{PuzzleMix}           &    \textcolor{gray}{-}   &  \textcolor{gray}{76.51}   & \textcolor{gray}{79.38} & \textcolor{gray}{80.89} &	\textcolor{gray}{82.46} \\
\stochmixupcutmix    & 74.88 & 75.49    & \textbf{79.25} & {81.05} &	81.21 \\ \midrule
\rowcolor{Gray}
\textbf{\ours} (ours)               & {74.99} & {75.68}    & \textbf{79.25} & {\textbf{81.07}} &	{81.38} \\
\rowcolor{Gray}
\textbf{\oursgaussian} (ours)      & {\textbf{75.75}} &  \textbf{76.15}  & 79.17 & 80.52 & \textbf{81.45} \\ \bottomrule
\end{tabular}
\end{table}

\begin{wraptable}{r}{0.5\linewidth}
\centering
\small
\vspace{-1.5em}
\caption{\small \textbf{ImageNet-1K classification.} Comparison of various MSDA methods on ResNet-50 architecture.}
\label{table:imagenet-classification}
\begin{tabular}{lc}
\toprule
Augmentation Method & Top-1 accuracy \\ \midrule
Vanilla (no MDSA)       &		75.68 (+0.00) \\
Mixup	        &		77.78 (+2.10) \\
CutMix	        &		78.04 (+2.36) \\
\stochmixupcutmix	&		{78.13} (+2.45)\\ \midrule
\rowcolor{Gray}
\textbf{\ours} (ours)	        &		\textbf{78.38} (+2.70) \\
\rowcolor{Gray}
\textbf{\oursgaussian} (ours)	&		{78.13} (+2.45) \\ \bottomrule
\end{tabular}
\vspace{-1em}
\end{wraptable}
\paragraph{Results on ImageNet-1K classification.}
\cref{table:imagenet-classification} shows the comparison of various MSDA methods on ImageNet-1K. 
We train ResNet-50 \cite{he2016deep} with various MSDA methods for 300 epochs using SGD optimizer with a learning rate $0.1$.
We set the hyper-parameter $\alpha$ for all methods except Mixup to $1$, while Mixup has $\alpha=0.8$.
We use $r=0.75$ for \ours. 
Here, we do not include PuzzleMix because it needs heavy additional computations to compute the input saliencies.
In the table, \ours shows the best performance, while \oursgaussian and \stochmixupcutmix show the second-best performances.
\revision{Evaluations on various robustness benchmarks are in \cref{appendix::additional-experiments}.}

\section{Conclusion}
We analyze MSDA by a unified theoretical framework. Our unified theoretical results show that any MSDA method behaves as a regularization on the input gradients and Hessians, where the degree of the regularization is controlled by the design choice of MSDA. We compare various MSDA methods in (1) regularization coefficient (2) regularized gradients (3) model performances in various scenarios with different pixel-level importance. We propose two simple MSDA methods, \ours and \oursgaussian, which leverage the benefits of Mixup and CutMix by their design. Our experimental results show that \ours and \oursgaussian outperform popular MSDA methods Mixup and CutMix. Furthermore, our methods show comparable or outperformed performances than the state-of-the-art MSDA method, \stochmixupcutmix, in CIFAR-100 and ImageNet classification tasks.

\bibliographystyle{abbrv}
\bibliography{Aug}

\appendix
\clearpage
\section*{Appendix}
We include additional materials in this document, including additional theoretical results (\cref{appendix::thm1pf}, \ref{appendix::n-mixup}, \ref{appendix::thm2pf}, \ref{appendix::thm3pf}, \ref{appendix:chidambaram}), experimental details (\cref{appendix::experimental}), additional experiments (\cref{appendix::additional-experiments}), and additional explanation for various MSDAs based on our analysis (\cref{appendix::hmix-regularizer}, \ref{appendix::other-msdas}).

\revision{\section*{Negative Societal Impacts} 
Our work may help the situation that data is insufficient, as other data augmentation papers do. We mainly proposed the theoretical aspect of MSDA, so there are no negative societal impacts. }

\section{Proof of Theorem~\ref{thm::MSDA-loss}}
\label{appendix::thm1pf}
For convenience, we restate Theorem~\ref{thm::MSDA-loss}.

\textbf{Theorem 1.}
Consider a loss function $l \in \cL$. We define $\tilde{D}_\lambda$ as $\frac{\alpha}{\alpha + \beta}\text{Beta}(\alpha + 1, \beta) + \frac{\beta}{\alpha + \beta}\text{Beta}(\beta + 1, \alpha)$. Assume that $\bE_{r_x \sim \cD_X}[r_x] = 0$. Then, we can re-write the general MSDA loss \eqref{eqn::MSDA-original-loss} as
\begin{align*}
L^{\text{MSDA}}_m(\theta) = L_m(\theta) + \sum_{i=1}^3 \cR_i^{(\text{MSDA})}(\theta) + \bE_{\lambda \sim \tilde{\cD}(\lambda)}\bE_M [(1-M)^\intercal \varphi(1-M) (1-M)], 
\end{align*}
where $\lim_{a \to 0}\varphi(a)= 0 $,
{\small
\begin{align*}
\begin{split}
    &\cR^{(\text{MSDA})}_1(\theta) = \frac{1}{m} \sum_{i=1}^m( h'(f_\theta(x_i))-y_i) \bE_{r_x \sim \cD_X}\left( \nabla f_\theta(x_i)\odot  (r_x - x_i)\right)^\intercal \bE_{\lambda \sim \tilde{D}_\lambda} \bE_M(1-M) 
    \\&\qquad \qquadß= \frac{1}{m} \sum_{i=1}^m( y_i - h'(f_\theta(x_i))) \left( \nabla f_\theta(x_i)^\intercal x_i\right) \bE_{\lambda \sim \tilde{D}_\lambda} (1-\lambda),
    \\
    &\cR^{(\text{MSDA})}_2(\theta) = \frac{1}{2m}\sum_{i=1}^m h''(f_\theta(x_i)) \bE_{\lambda \sim \tilde{D}_\lambda} \fG(\cD_X, x_i, f, M)  ,  
    \\
    &\cR^{(\text{MSDA})}_3(\theta) = \frac{1}{2m}\sum_{i=1}^m(h'(f_\theta(x_i))-y_i) \bE_{\lambda \sim \tilde{D}_\lambda} \fH(\cD_X, x_i, f, M),
\end{split}
\end{align*}
and 
\begin{align*}
\begin{split}
    \fG(\cD_X, &x_i, f, M) = \bE_M (1-M)^\intercal \bE_{r_x \sim \cD_X} \left(\nabla f(x_i) \odot (r_x- x_i) \left(\nabla f(x_i) \odot (r_x- x_i)\right)^\intercal\right)(1-M)
    \\
    &= \sum_{j,k \in \text{coord}}a_{jk} \partial_j f_\theta(x_i)\partial_k f_\theta(x_i)\left(\bE_{r_x \sim \cD_X}[r_{xj}r_{xk}] + x_{ij}x_{ik}\right),
    \\
    \fH(\cD_X, &x_i, f, M) = \bE_{r_x \sim \cD_X}\bE_M (1-M)^\intercal  \left(\nabla^2f_\theta(x_i) \odot \left((r_x- x_i)(r_x- x_i)^\intercal\right)\right)(1-M)
    \\
    &= \sum_{j,k \in \text{coord}} a_{jk} \left(\bE_{r_x \sim \cD_X}[r_{xj}r_{xk}\partial^2_{jk}f_\theta(x_i)] + x_{ij}x_{ik}\partial^2_{jk}f_\theta(x_i)\right),
\end{split}
\end{align*}
where

\begin{align*}
    a_{jk} := \bE_M[(1-M_j)(1-M_k)].
\end{align*}
}

\begin{proof}[Proof of Theorem~\ref{thm::MSDA-loss}]
Due to the assumption of the theorem, we can rewrite the empirical loss for the non-augmented population as 
\[
L_m(\theta) = \frac{1}{m}\sum_{i=1}^m l(\theta, z_i) = \frac{1}{m}\sum_{i=1}^n [h(f_\theta(x_i)) - y_i f_\theta(x_i)].
\]
Similarly, we can rewrite the MSDA loss as 
\begin{align*}
L^{\text{MSDA}}_m(\theta) &= \frac{1}{m^2}  \sum_{i,j=1}^m\bE_{\lambda \sim \cD_\lambda} \bE_Ml(\theta, \msda{z}{i,j}(\lambda, 1-\lambda))
\\
&= \frac{1}{m^2}  \sum_{i,j=1}^m\bE_{\lambda \sim  \text{Beta}(\alpha, \beta)} \bE_M[h(f_\theta( \msda{x}{i,j}(M, 1-M))) - \msda{y}{i,j}(\lambda, 1-\lambda) f_\theta(\msda{x}{i,j}(M, 1-M))].    
\end{align*}

Putting the definition of $\msda{z}{i,j}(\lambda, 1-\lambda)$ to the  equation above, we have 
\begin{align*}
    L^{\text{MSDA}}_m(\theta) &= \frac{1}{m^2} \sum_{i,j=1}^m \Biggl( \bE_{\lambda \sim  \text{Beta}(\alpha, \beta)}  \bE_M  \biggl( \lambda (h(f_\theta( \msda{x}{i,j}(M, 1-M))) - y_i f_\theta(\msda{x}{i,j}(\lambda, 1-\lambda)))
    \\
    & +  (1-\lambda) (h(f_\theta( \msda{x}{i,j}(M, 1-M))) - y_j f_\theta(\msda{x}{i,j}(M, 1-M)))\biggr)\Biggr)
    \\    
    &= \frac{1}{m^2}\sum_{i,j=1}^m \Biggl( \bE_{\lambda \sim  \text{Beta}(\alpha, \beta)} \bE_{B \sim \text{Bin}(\lambda)} \bE_M \biggl(  B (h(f_\theta( \msda{x}{i,j}(M, 1-M))) - y_i f_\theta(\msda{x}{i,j}(M, 1-M)))
    \\
    & +  (1-B) (h(f_\theta( \msda{x}{i,j}(M, 1-M))) - y_j f_\theta(\msda{x}{i,j}(M, 1-M)))\biggr)\Biggr).
\end{align*}
Note that $\lambda \sim \text{Beta}(\alpha, \beta)$ and $B | \lambda \sim \text{Bin}(\lambda)$. By conjugacy, we can write the joint distribution of $(\lambda, B)$ as 
\begin{align*}
    B \sim \text{Bin}\left(\frac{\alpha}{\alpha + \beta} \right), \qquad \lambda| B \sim \text{Beta}(\alpha + B, \beta + 1 - B). 
\end{align*}
Therefore, we have
\begin{align}
    L^{\text{MSDA}}_m(\theta) & = \frac{1}{m^2} \sum_{i,j=1}^m \biggl( \bE_{\lambda \sim  \text{Beta}(\alpha+1, \beta)} \bE_M \frac{\alpha}{\alpha + \beta} (h(f_\theta( \msda{x}{i,j}(\lambda, 1-\lambda))) - y_i f_\theta(\msda{x}{i,j}(\lambda, 1-\lambda))) \nonumber
    \\
    & +  \bE_{\lambda \sim  \text{Beta}(\alpha, \beta +1)} \bE_M \frac{\beta}{\alpha + \beta} (h(f_\theta( \msda{x}{i,j}(\lambda, 1-\lambda))) - y_j f_\theta(\msda{x}{i,j}(\lambda, 1-\lambda)))\biggr)\nonumber
    \\
    &= \frac{1}{m} \sum_{i=1}^n \bE_{\lambda \sim \tilde{\cD}(\lambda) } \bE_{r_{x}\sim \cD_x} \bE_M \left[h\left(f_\theta (M \odot x_i + (1-M) \odot r_x)\right) - y_i f_\theta (M \odot x_i + (1-M) \odot r_x)\right] \label{eqn::thm1-cutmix}
    \\
    &= \frac{1}{m} \sum_{i=1}^n \bE_{\lambda \sim \tilde{\cD}(\lambda) } \bE_{r_{x}\sim \cD_x} \bE_M l (\theta, \hat{z_i}), \label{eqn::thm1-reformulation}
\end{align}
where $\hat{z_i} = (M \odot x_i + (1-M) \odot r_x, y_i)$. 

Let $N = 1-M$. By defining $\phi_i(N) = h\left(f_\theta (x_i + N \odot (r_x - x_i))\right)   - y_i f_\theta ( x_i + N \odot (r_x - x_i))$ and applying Taylor expansion, we have 
\begin{align}
  \phi_i(N) = \phi_i(0) + \nabla_N \phi_i(0)^\intercal N + \frac{1}{2} N^\intercal \nabla^2_N \phi_i(0) N + N^\intercal \varphi(N) N,  \label{eqn::thm1-1}
\end{align}
where $\lim_{N \to 0}\varphi(N) = 0$. 
Firstly, we calculate $\phi_i(0)$ by 
\begin{align}
    \phi_i(0) = h(f_\theta(x_i)) - y_i f_\theta(x_i). \label{eqn::thm1-2}
\end{align}
Second, we calculate $\nabla_N \phi_i(0)$ by 
\begin{align*}
    \frac{\partial\phi_i(N)}{\partial N_k} &= \left(h'\left(f_\theta \left(x_i + N \odot (r_x - x_i)\right)\right) - y_i\right) \fracpartial{f_\theta}{x_{ik}}\left(x_i + N \odot (r_x - x_i)\right)(r_{xk} - x_{ik}),
\end{align*}
where we denote $N_k$ as the $k$th element of $N$, $x_{ik}$ as the $k$th element of $x_i$, and $r_{xk}$ as the $k$th element of $r_x$. 
Therefore, we have
\begin{align}
    \nabla_N \phi_i(0)^\intercal N &= (h'(f_\theta(x_i)) - y_i)\sum_{k} \left(\fracpartial{f_\theta}{x_{ik}}(x_i)(r_{xk} - x_{ik})\right)N_k \nonumber
    \\
    &= (h'(f_\theta(x_i)) - y_i) \left(\nabla f \odot (r_x- x_i)\right) \cdot N. \label{eqn::thm1-3}
\end{align} 
Finally, we calculate $\nabla^2_N \varphi_i(\vec{0})^T$ by 
\begin{align*}
    \frac{\partial^2\phi_k(N)}{\partial N_k \partial N_j} &= \frac{\partial}{\partial N_j} \left(\left(h'\left(f_\theta \left(x_i + N \odot (r_x - x_i)\right)\right) - y_i\right) \fracpartial{f_\theta}{x_{ik}}\left(x_i + N \odot (r_x - x_i)\right)(r_{xk} - x_{ik})\right)
    \\
    &= h''\left(f_\theta \left(x_i + N \odot (r_x - x_i)\right)\right)
    \\
    &\qquad \times \fracpartial{f_\theta}{x_{ik}}\left(x_i + N \odot (r_x - x_i)\right)(r_{xk} - x_{ik}) \fracpartial{f_\theta}{x_{ij}}\left(x_i + N \odot (r_x - x_i)\right)(r_{xj} - x_{ij}) 
    \\
    &+ \left(h'\left(f_\theta \left(x_i + N \odot (r_x - x_i)\right)\right) - y_i\right)
    \\
    &\qquad \times \frac{\partial^2 f_\theta}{\partial x_{ik} \partial x_{ij}}\left(x_i + N \odot (r_x - x_i)\right)(r_{xk}- x_{ik})(r_{xj} - x_{ij}).
\end{align*}
Therefore, we have 
\begin{align}
    \frac{1}{2} N^\intercal \nabla^2_N \phi_i(0) N &=\frac{1}{2} h''\left(f_\theta (x_i)\right) \sum_{k, j} \left(\fracpartial{f_\theta}{x_{ik}}(x_i)(r_{xk} - x_{ik})\fracpartial{f_\theta}{x_{ij}}(x_i)(r_{xj} - x_{ij}) N_kN_j\right) \nonumber
    \\
    &\qquad  + \frac{1}{2}\left(h'\left(f_\theta (x_i)\right) - y_i \right) \sum_{k,j} \frac{\partial^2 f_\theta}{\partial x_{ik} \partial x_{ij}}\left(x_i\right)(r_{xk}- x_{ik})(r_{xj} - x_{ij})N_k N_j \nonumber
    \\
    &= \frac{1}{2} h''\left(f_\theta (x_i)\right) N^\intercal\left( \left(\nabla f \odot (r_x- x_i)\right) \left(\nabla f \odot (r_x- x_i)\right)^\intercal\right) N  \nonumber
    \\
    &\qquad +  \frac{1}{2}\left(h'\left(f_\theta (x_i)\right) - y_i \right)N^\intercal \left(\nabla^2f_\theta(x_i) \odot \left((r_x- x_i)(r_x- x_i)^\intercal\right) \right) N  .\label{eqn::thm1-4}
\end{align}
Applying \eqref{eqn::thm1-2} - \eqref{eqn::thm1-4} to \eqref{eqn::thm1-1},
\begin{align}
 \phi_i(N) &= \left(h(f_\theta(x_i)) - y_i f_\theta(x_i)\right) + (h'(f_\theta(x_i)) - y_i) \left(\nabla f \odot (r_x- x_i)\right) \cdot N  \nonumber
  \\
  &+ \frac{1}{2} h''\left(f_\theta (x_i)\right) N^\intercal\left( \left(\nabla f \odot (r_x- x_i)\right) \left(\nabla f \odot (r_x- x_i)\right)^\intercal\right) N  \nonumber
    \\
    &+  \frac{1}{2}\left(h'\left(f_\theta (x_i)\right) - y_i \right)N^\intercal \left(\nabla^2f_\theta(x_i) \odot \left((r_x- x_i)(r_x- x_i)^\intercal\right) \right) N + N^\intercal \varphi(N) N  \label{eqn::thm1-5} 
\end{align}
Plugging \eqref{eqn::thm1-5} to \eqref{eqn::thm1-cutmix}, we conclude 
\begin{align*}
    L^{\text{MSDA}}_m(\theta) & =\frac{1}{m} \sum_{i=1}^n \bE_{\lambda \sim \tilde{\cD}(\lambda) } \bE_{r_{x}\sim \cD_x} \bE_M \phi(1-M)
    \\
    &= L_m(\theta) + \cR_1(\theta) + \cR_2 (\theta) + \cR_3(\theta) +  \bE_{\lambda \sim \tilde{\cD}(\lambda)}\bE_M [(1-M)^\intercal \varphi(1-M) M],
\end{align*}
where
\begin{align*}
    &\cR_1(\theta) =\frac{1}{m} \sum_{i=1}^m(h'(f_\theta(x_i)) - y_i) \left( \nabla f_\theta(x_i) \odot \bE_{r_x \sim \cD_X}[r_x - x_i]\right) \bE_{\lambda \sim \tilde{D}_\lambda} \bE_M(1-M),
    \\
    &\cR_2(\theta) = \frac{1}{2m}\sum_{i=1}^m h''(f_\theta(x_i)) \bE_{\lambda \sim \tilde{D}_\lambda} \bE_M(1-M)^\intercal  \bE_{r_x \sim \cD_X} \left[\nabla f(x_i) \odot (r_x- x_i) \left(\nabla f(x_i) \odot (r_x- x_i)\right)^\intercal\right] (1-M) ,    
    \\
    &\cR_3(\theta) = \frac{1}{2m}\sum_{i=1}^m \bE_{\lambda \sim \tilde{D}_\lambda} \bE_M(1-M)^\intercal  \bE_{r_x \sim \cD_X} \left[\nabla^2f_\theta(x_i) \odot \left((r_x- x_i)(r_x- x_i)^\intercal\right) \right] (1-M).  
\end{align*}
\end{proof}
\clearpage

\section{Extension of Mixup: $n$-Mixup}
\label{appendix::n-mixup}
In this section, due to notational complexity, we give the approximate loss function of the $n$-sample Mixup ($n$-Mixup). The same analysis can be applied to $n$-sample mixing strategy. 
We will define $n$-Mixup as followings. Mixup from the $\ffi = (i_1, i_2, \dots, i_n)$th samples with $\flam = (\lambda_1, \lambda_2, \dots, \lambda_n)$ which is drawn from $\cD_\Lambda$ (mainly Dirichlet distribution), is defined as $\tilde{z}_{\ffi} = \sum_{k=1}^n \lambda_k z_{i_k}$. Similarly, we can define the $n$-Mixup loss as 
\begin{align*}
    L^{\text{n-mixup}}_m(\theta) &= \bE_{ \ffi \sim \text{Unif}([m])}{\bE_{\flam \sim \cD_\Lambda} l(\theta, \tilde{z}_{\ffi}(\flam))} = \frac{1}{m^n}\sum_{\ffi}\bE_{\flam \sim \cD_\Lambda} l(\theta, \tilde{z}_{\ffi}(\flam)).
\end{align*}
Throughout this section, we consider $\cD_\Lambda$ as Dirichlet distribution (i.e. $\cD_\Lambda = $ Dir$(\falp)$ = Dir$(\alpha_1, \alpha_2, \dots, \alpha_n)$), which is the natural extension of Beta distribution. 
\begin{theorem}
\label{thm::reg-nmix}
Consider the loss function in $l \in \cL$. Then, we can rewrite the $n$-Mixup loss as 
\begin{align*}
    L^{\text{n-mix}}_m(\theta) & = \frac{1}{m} \sum_{k=1}^n \bE_{\flam \sim \tilde{\cD}(\Lambda) } \bE_{r_{x,2}, \cdots, r_{x, n} \sim \cD_x} \varphi_k(\lambda_2, \cdots, \lambda_n)
    \\
    &= L_m(\theta) + \cR_1(\theta) + \cR_2 (\theta) + \cR_3(\theta) +  \bE_{\flam \sim \tilde{\cD}(\Lambda)}[o(\|(\lambda_2, \cdots, \lambda_n)\|^2)],
\end{align*}
where 
\begin{align*}
    \cR_1(\theta) = \frac{\bE_{\flam \sim \tilde{\cD}_{\Lambda}}[1- \lambda_1]}{m} \sum_{i=1}^m (h'(f_\theta(x_i)) - y_i) \nabla f_\theta(x_i)^T \bE_{r_x \sim \cD_X}[r_x- x_i],
\end{align*}
\begin{align*}
    \cR_2(\theta) &= \frac{\bE_{\flam \sim \tilde{\cD}_{\Lambda}}[\sum_{j=2}^m \lambda_j^2 ]}{2m} \sum_{i=1}^m h''(f_\theta(x_i))  \nabla f_\theta(x_i)^T \bE_{r_x \sim \cD_X}[(r_x- x_i)(r_x- x_i)^T]\nabla_\theta f(x_i)
    \\
    &+ \frac{\bE_{\flam \sim \tilde{\cD}_{\Lambda}}[(1-\lambda_1)^2 - \sum_{j=2}^m \lambda_j^2 ]}{2m} \sum_{i=1}^m h''(f_\theta(x_i))  \nabla f_\theta(x_i)^T \bE_{r_x \sim \cD_X}[(r_x- x_i)] \bE_{r_x \sim \cD_X}[(r_x- x_i)^T]\nabla_\theta f(x_i), 
\end{align*}
\begin{align*}
    \cR_3(\theta) &= \frac{\bE_{\flam \sim \tilde{\cD}_{\Lambda}}[\sum_{j=2}^m \lambda_j^2 ]}{2m} \sum_{i=1}^m (h'(f_\theta(x_i)) - y_i) \bE_{r_x \sim \cD_X}[(r_x- x_i) \nabla^2 f_\theta (x_i) (r_x- x_i)^T]
    \\
    &+ \frac{\bE_{\flam \sim \tilde{\cD}_{\Lambda}}[(1-\lambda_1)^2 - \sum_{j=2}^m \lambda_j^2 ]}{2m} \sum_{i=1}^m (h'(f_\theta(x_i)) - y_i) \bE_{r_x \sim \cD_X}[(r_x- x_i)] \nabla^2 f_\theta (x_i) \bE_{r_x \sim \cD_X}[(r_x- x_i)^T].  
\end{align*}
\end{theorem}
As Mixup's approximate loss function \cite{zhang2020does} or Theorem~\ref{thm::MSDA-loss}, $n$-Mixup also regularizes $\nabla f$ and $\nabla^2 f$. 
\begin{proof}
Due to the assumption of the theorem, we can rewrite the empirical loss for the non-augmented population as 
\[
L_m(\theta) = \frac{1}{m}\sum_{i=1}^m l(\theta, z_i) = \frac{1}{m}\sum_{i=1}^n [h(f_\theta(x_i)) - y_i f_\theta(x_i)].
\]
Similarly, we can rewrite the $n$-Mixup loss as 
\[
L^{\text{n-mix}}_m(\theta) = \frac{1}{m^n}\sum_{\ffi}\bE_{\flam \sim \cD_\Lambda} l(\theta, \tilde{z}_{\ffi}(\flam))= \frac{1}{m^n}\bE_{\flam \sim  \text{Dir}(\falp)} \sum_{\ffi}[h(f_\theta( \tilde{x}_\ffi(\flam))) - \tilde{y}_\ffi(\flam) f_\theta(\tilde{x}_\ffi(\flam))]. 
\]
Putting the definition of $\tilde{x}_\ffi$ and $\tilde{y}_\ffi$ to the equation above, we have 
\begin{align*}
    L^{\text{n-mix}}_m(\theta) &= \frac{1}{m^n}\bE_{\flam \sim  \text{Dir}(\falp)} \sum_{\ffi} \sum_{k=1}^n \lambda_k (h(f_\theta( \tilde{x}_\ffi(\flam))) - y_k f_\theta(\tilde{x}_\ffi(\flam)))
    \\
    & = \frac{1}{m^n}\sum_{\ffi}\bE_{\flam \sim  \text{Dir}(\falp)} \bE_{\fbet \sim \text{Mult}(\flam)}  \sum_{k=1}^n \beta_k (h(f_\theta( \tilde{x}_\ffi(\flam))) - y_k f_\theta(\tilde{x}_\ffi(\flam))),
\end{align*}
where $\fbet = (\beta_1, \beta_2, \cdots, \beta_n)$ and $\text{Mult}(\flam)$ is multinomial distribution. Note that $\flam \sim \text{Dir}(\falp)$ and $\fbet | \flam \sim \text{Mult}(\flam)$. By conjugacy, we can write the joint distribution of $(\falp, \fbet)$ as 
\begin{align*}
    \fbet \sim \text{Mult}\left(\frac{\falp}{\sum \alpha_i} \right), \qquad \flam| \fbet \sim \text{Dir}(\falp + \fbet). 
\end{align*}
Therefore, 
\begin{align}
    L^{\text{n-mix}}_m(\theta) & = \frac{1}{m^n}\sum_{\ffi}\sum_{k=1}^n \frac{\alpha_k}{\sum \alpha_i} \bE_{\flam \sim \text{Dir}(\alpha_1, \alpha_2, \dots, \alpha_k + 1, \dots, \alpha_n)}  (h(f_\theta( \tilde{x}_\ffi(\flam))) - y_k f_\theta(\tilde{x}_\ffi(\flam))) \nonumber
    \\
    &= \frac{1}{m} \sum_{k=1}^n \bE_{\flam \sim \tilde{\cD}(\Lambda) } \bE_{r_{x,2}, \cdots, r_{x, n} \sim \cD_x} \left[h\left(f_\theta \left(\lambda_1 x_k + \sum_{j=2}^n \lambda_j r_{x,j}\right)\right) - y_k f_\theta\left(\lambda_1 x_k + \sum_{j=2}^n \lambda_j r_{x,j}\right)\right], \label{eqn::thm1-nmix}
\end{align}
where $\tilde{D}_{\Lambda} = \sum_{l = 1}^n \frac{\alpha_l}{\sum \alpha_i} \text{Dir}(\alpha_l + 1, \alpha_{l+1}, \dots, \alpha_{l+n-1})$ and $\cD_x$ is the empirical distribution induced by training samples. We regard the index with mod $n$. 
Defining $\varphi_k(\cdot)$ as  $$\varphi_k(\lambda_2, \cdots, \lambda_n) = h\left(f_\theta \left((1- \sum_{j=2}^n \lambda_j)x_k + \sum_{j=2}^n \lambda_j r_{x,j}\right)\right) - y_k f_\theta\left((1- \sum_{j=2}^n \lambda_j) x_k + \sum_{j=2}^n \lambda_j r_{x,j}\right), $$
we can use twice the differentiability of $f(\cdot)$ and $h(\cdot)$, so we have
\begin{align}
  \varphi_k(\lambda_2, \cdots, \lambda_n) = \varphi_k(\vec{0}) + \nabla \varphi_k(\vec{0})^T (\lambda_2, \cdots, \lambda_n) + \frac{1}{2} (\lambda_2, \cdots, \lambda_n)^T \nabla^2 \varphi_k(\vec{0})^T (\lambda_2, \cdots, \lambda_n)^T + o(\|(\lambda_2, \cdots, \lambda_n)\|^2).  \label{eqn::thmn1-1}
\end{align}
Firstly, we calculate $\varphi_k(\vec{0})$ by 
\begin{align}
    \varphi_k(\vec{0}) = h(f_\theta(x_k)) - y_k f_\theta(x_k). \label{eqn::thmn1-2}
\end{align}
Second, we calculate $\nabla \varphi_k(\vec{0})$ by 
\begin{align*}
    \frac{\partial\varphi_k(\lambda_2, \cdots, \lambda_n)}{\partial \lambda_i} &= h'\left(f_\theta \left((1- \sum_{j=2}^n \lambda_j)x_k + \sum_{j=2}^n \lambda_j r_{x,j}\right)\right)f_\theta' \left((1- \sum_{j=2}^n \lambda_j)x_k + \sum_{j=2}^n \lambda_j r_{x,j}\right)(r_{x, i} - x_k)
    \\
    &\qquad - y_k f_\theta' \left((1- \sum_{j=2}^n \lambda_j)x_k + \sum_{j=2}^n \lambda_j r_{x,j}\right)(r_{x, i} - x_k) 
\end{align*}
Therefore, we have
\begin{align}
    \frac{\partial\varphi_k(\lambda_2, \cdots, \lambda_n)}{\partial \lambda_i} \Big|_{(\lambda_2, \cdots, \lambda_n) = \vec{0}} = (h'(f_\theta(x_k)) - y_i)\nabla f_\theta(x_k)^T (r_{x,i} - x_k). \label{eqn::thmn1-3}
\end{align} 
Finally, we calculate $\nabla^2 \varphi_i(\vec{0})^T$ by 
{\small
\begin{align*}
    \frac{\partial^2\varphi_k(\lambda_2, \cdots, \lambda_n)}{\partial \lambda_i \partial \lambda_s} &= \frac{\partial}{\partial \lambda_s} \Biggl( h'\left(f_\theta \left((1- \sum_{j=2}^n \lambda_j)x_k + \sum_{j=2}^n \lambda_j r_{x,j}\right)\right)f_\theta' \left((1- \sum_{j=2}^n \lambda_j)x_k + \sum_{j=2}^n \lambda_j r_{x,j}\right)(r_{x, i} - x_k)
    \\
    &\qquad - y_k f_\theta' \left((1- \sum_{j=2}^n \lambda_j)x_k + \sum_{j=2}^n \lambda_j r_{x,j}\right)(r_{x, i} - x_k)\Biggr)
    \\
    &= h''\left(f_\theta \left((1- \sum_{j=2}^n \lambda_j)x_k + \sum_{j=2}^n \lambda_j r_{x,j}\right)\right)\left[f_\theta' \left((1- \sum_{j=2}^n \lambda_j)x_k + \sum_{j=2}^n \lambda_j r_{x,j}\right)(r_{x, i} - x_k)\right]
    \\
    &\quad \left[f_\theta' \left((1- \sum_{j=2}^n \lambda_j)x_k + \sum_{j=2}^n \lambda_j r_{x,j}\right)(r_{x, s} - x_k)\right]
    \\
    &\qquad + h'\left(f_\theta \left((1- \sum_{j=2}^n \lambda_j)x_k + \sum_{j=2}^n \lambda_j r_{x,j}\right)\right)(r_{x, i} - x_k)^T \nabla^2 f_\theta\left((1- \sum_{j=2}^n \lambda_j)x_k + \sum_{j=2}^n \lambda_j r_{x,j}\right)(r_{x, s} - x_k)
    \\
    & \qquad -y_k (r_{x, i} - x_k)^T \nabla^2 f_\theta\left((1- \sum_{j=2}^n \lambda_j)x_k + \sum_{j=2}^n \lambda_j r_{x,j}\right)(r_{x, s} - x_k).
\end{align*}
}
Therefore,
\begin{align}
    \frac{\partial^2\varphi_k(\lambda_2, \cdots, \lambda_n)}{\partial \lambda_i \partial \lambda_s}\Big|_{(\lambda_2, \cdots, \lambda_n) = \vec{0}}  &= \left(h''\left(f_\theta (x_k) - y_k\right)\right)\left(\nabla f_\theta \left(x_k\right)^T(r_{x, i} - x_k)\right)\left(\nabla f_\theta \left(x_k\right)^T(r_{x, s} - x_k)\right) \nonumber
    \\
    &\qquad -y_k (r_{x, i} - x_k)^T \nabla^2 f_\theta\left(x_k\right)(r_{x, s} - x_k).\label{eqn::thmn1-4}
\end{align}
Applying \eqref{eqn::thmn1-2} - \eqref{eqn::thmn1-4} to \eqref{eqn::thmn1-1}, 
\begin{align}
 &\varphi_k(\lambda_2, \cdots, \lambda_n) = \left(h(f_\theta(x_k)) - y_k f_\theta(x_k)\right) + \sum_{j=2}^n (h'(f_\theta(x_k)) - y_i)\nabla f_\theta(x_k)^T (r_{x,i} - x_k) \lambda_j \nonumber
  \\
  &\qquad + \frac{1}{2} \sum_{i, s= 2}^n \biggl(\left(h''\left(f_\theta (x_k) - y_k\right)\right)\left(\nabla f_\theta \left(x_k\right)^T(r_{x, i} - x_k)\right)\left(\nabla f_\theta \left(x_k\right)^T(r_{x, s} - x_k)\right) \nonumber
  \\ &\qquad \qquad -y_k (r_{x, i} - x_k)^T \nabla^2 f_\theta\left(x_k\right)(r_{x, s} - x_k) \biggr) \lambda_i \lambda_s+ o(\|(\lambda_2, \cdots, \lambda_n)\|^2).  \label{eqn::thmn1-5} 
\end{align}
Plugging \eqref{eqn::thmn1-5} to \eqref{eqn::thm1-nmix}, 
\begin{align*}
    L^{\text{n-mix}}_m(\theta) & = \frac{1}{m} \sum_{k=1}^n \bE_{\flam \sim \tilde{\cD}(\Lambda) } \bE_{r_{x,2}, \cdots, r_{x, n} \sim \cD_x} \varphi_k(\lambda_2, \cdots, \lambda_n)
    \\
    &= L_m(\theta) + \cR_1(\theta), + \cR_2 (\theta) + \cR_3(\theta) +  \bE_{\flam \sim \tilde{\cD}(\Lambda)}[o(\|(\lambda_2, \cdots, \lambda_n)\|^2)],
\end{align*}
where 
\begin{align*}
    \cR_1(\theta) = \frac{\bE_{\flam \sim \tilde{\cD}_{\Lambda}}[1- \lambda_1]}{m} \sum_{i=1}^m (h'(f_\theta(x_i)) - y_i) \nabla f_\theta(x_i)^T \bE_{r_x \sim \cD_X}[r_x- x_i],
\end{align*}
\begin{align*}
    \cR_2(\theta) &= \frac{\bE_{\flam \sim \tilde{\cD}_{\Lambda}}[\sum_{j=2}^m \lambda_j^2 ]}{2m} \sum_{i=1}^m h''(f_\theta(x_i))  \nabla f_\theta(x_i)^T \bE_{r_x \sim \cD_X}[(r_x- x_i)(r_x- x_i)^T]\nabla_\theta f(x_i)
    \\
    &+ \frac{\bE_{\flam \sim \tilde{\cD}_{\Lambda}}[(1-\lambda_1)^2 - \sum_{j=2}^m \lambda_j^2 ]}{2m} \sum_{i=1}^m h''(f_\theta(x_i))  \nabla f_\theta(x_i)^T \bE_{r_x \sim \cD_X}[(r_x- x_i)] \bE_{r_x \sim \cD_X}[(r_x- x_i)^T]\nabla_\theta f(x_i), 
\end{align*}
\begin{align*}
    \cR_3(\theta) &= \frac{\bE_{\flam \sim \tilde{\cD}_{\Lambda}}[\sum_{j=2}^m \lambda_j^2 ]}{2m} \sum_{i=1}^m (h'(f_\theta(x_i)) - y_i) \bE_{r_x \sim \cD_X}[(r_x- x_i) \nabla^2 f_\theta (x_i) (r_x- x_i)^T]
    \\
    &+ \frac{\bE_{\flam \sim \tilde{\cD}_{\Lambda}}[(1-\lambda_1)^2 - \sum_{j=2}^m \lambda_j^2 ]}{2m} \sum_{i=1}^m (h'(f_\theta(x_i)) - y_i) \bE_{r_x \sim \cD_X}[(r_x- x_i)] \nabla^2 f_\theta (x_i) \bE_{r_x \sim \cD_X}[(r_x- x_i)^T].  
\end{align*}
\end{proof}
\section{Adversarial Robustness of MSDA}
\label{appendix::thm2pf}
Let us scrutinize adversarial robustness in MSDA. We adopt the logistic loss, so $l(\theta, z) = \log(1+ \exp(f_\theta(x))) - yf_\theta(x)$ where $y \in \{0,1\}$. Define $g(s) = e^s/(1+e^s)$. As \cite{zhang2020does}, we scrutinize the logistic regression with $f_\theta(x)$ as ReLU or leaky-ReLU network. Then, we have $f_\theta(x) = \nabla f_\theta(x_i)^\intercal x_i$ and $\nabla^2 f_\theta(x_i) = 0$. We consider the adversarial loss with $l_2$ attack of size $\epsilon \sqrt{d}$ ($d$ is the dimension of $\theta$), that is, $L_m^{\text{adv}}(\theta) = \frac{1}{m} \sum_{i=1}^m \max_{\norm{\delta_i}_2 \leq \epsilon \sqrt{d}} l(\theta, (x_i + \delta_i, y_i))$.

\textbf{Theorem 3.} \emph{
In MSDA, we suppose that $f_\theta(x) = \nabla f_\theta(x_i)^\intercal x_i, \nabla^2 f_\theta(x_i) = 0$ and there exists a constant $c_x >0$ that $\norm{x_i}_2 \leq c_x \sqrt{d}$ for all $i \in [m]$. Then for any $\theta \in \Theta$, we have 
\begin{align*}
    \tilde{L}_m^{\text{(MSDA)}} \geq \frac{1}{m} \sum_{i=1}^m \tilde{l}_{\text{adv}}(\epsilon_i \sqrt{d}, z) \geq \frac{1}{m} \sum_{i=1}^m \tilde{l}_{\text{adv}}(\epsilon_{\text{cut}} \sqrt{d}, z),
\end{align*}
where $\epsilon_{cut} = c_x \min\left( \min_i |\cos(\nabla f_\theta(x_i), x_i)| \bE_{\lambda \sim \tilde{\cD}_\lambda}[1-\lambda], \min_i \bE_{\lambda \sim \tilde{D}_\lambda, M}s(M, x_i) \right) $ and $s(M, x_i) = \frac{ (\sqrt{1-M} \odot \nabla f(x_i))^T (\sqrt{1-M} \odot x_i)}{\norm{\nabla f(x_i)}_2\norm{x_i}_2}. $
}

This bound appears to be fertile at first glance. However, as $\norm{f_\theta(x_i)}_2$ increases after training accuracy reaches 100\%, the logistic loss decreases. Therefore, due to $\norm{f_\theta(x_i)}_2 = \norm{\nabla f_\theta(x_i)^\intercal x_i}_2 = \norm{\nabla f_\theta(x_i)}_2\norm{x_i}_2 \cos(\nabla f_\theta(x_i), x_i)$, $\cos(\nabla f_\theta(x_i), x_i)$ would be larger. Furthermore, under the CutMix case, since $M$ distribution is relatively uniform, $s(M, x_i)$ will be similar to $\min_i |\cos(\nabla f_\theta(x_i), x_i)| \bE_{\lambda \sim \tilde{\cD}_\lambda}[1-\lambda]$ \cite{zhang2020does}. Moreover, empirical results support Mixup's or Cutmix's adversarial robustness \cite{zhang2017mixup, yun2019cutmix, rebuffi2021data}. 
\begin{proof}[proof of \cref{thm::MSDA-robustness}]
\begin{fact}
\emph{\cite{zhang2020does}} The second order Taylor approximation of $L_m^{\text{adv}}(\theta)$ is $\frac{1}{m}\sum_{i=1}^m \tilde{l}_{\text{adv}}(\epsilon\sqrt{d}, z)$ where fore any $\eta>0, x\in\bR^d$ and $y \in \{0,1\}$, 
$$\tilde{l}_{\text{adv}}(\eta, z) = l(\theta, z) + \eta |g(f_\theta(x)) - y| \norm{\nabla f_\theta(x)}_2 + \frac{\eta^2d}{2}\cdot |h''(f_\theta(x))| \cdot \norm{\nabla f_\theta(x)}_2^2. $$
\end{fact}
We set $\bE_{r_x}(r_x) = 0$ by parallel translation. We define $$s(M, x_i) = \frac{ (\sqrt{1-M} \odot \nabla f(x_i))^T (\sqrt{1-M} \odot x_i)}{\norm{\nabla f(x_i)}_2\norm{x_i}_2}. $$

For every MSDA, we have $\cR_1$ as 
\begin{align*}
    \cR_1(\theta) = \frac{\bE_\lambda(1-\lambda)}{m} \sum_{i=1}^m (y_i - g(f_\theta(x_i)))f_\theta(x_i),
\end{align*}
and since $\theta \in \Theta$, we have $(y_i - g(f_\theta(x_i)))f_\theta(x_i) \geq 0$. Therefore,
\begin{align*}
    \cR_1(\theta) &= \frac{\bE_\lambda(1-\lambda)}{m} \sum_{i=1}^m |y_i - g(f_\theta(x_i))| |f_\theta(x_i)|
    \\
    &=\frac{\bE_\lambda(1-\lambda)}{m} \sum_{i=1}^m |y_i - g(f_\theta(x_i))| \norm{\nabla f_\theta(x_i)}_2 \norm{x_i}_2 |\cos(\nabla f_\theta(x_i), x_i)|
    \\
    &\geq \frac{\min_i{\norm{x_i}_2} \min_i{|\cos(\nabla f_\theta(x_i), x_i)|} \bE_\lambda(1-\lambda)}{m} \sum_{i=1}^m |y_i - g(f_\theta(x_i))| \norm{\nabla f_\theta(x_i)}_2.
\end{align*}

Moreover, we can eliminate $\cR_3$ since $\nabla^2 f_\theta = 0$. So, we only focus on $\cR_2$ term. We have
\begin{align*}
    \cR^{(\text{MSDA})}_2(\theta) &= \frac{1}{2m}\sum_{i=1}^m h''(f_\theta(x_i)) \bE_{\lambda \sim \tilde{D}_\lambda} \bE_M (1-M)^\intercal \bE_{r_x \sim \cD_X} \left(\nabla f(x_i) \odot (r_x- x_i) \left(\nabla f(x_i) \odot (r_x- x_i)\right)^\intercal\right)(1-M)
    \\
    &\geq \frac{1}{2m}\sum_{i=1}^m h''(f_\theta(x_i)) \bE_{\lambda \sim \tilde{D}_\lambda} \bE_M (1-M)^\intercal  \left((\nabla f(x_i) \odot x_i) \left(\nabla f(x_i) \odot x_i\right)^\intercal\right)(1-M)
    \\
    &= \frac{1}{2m}\sum_{i=1}^m |g(f_\theta(x_i))(1-g(f_\theta(x_i)))| \bE_{\lambda \sim \tilde{D}_\lambda} \bE_M ((\sqrt{1-M} \odot \nabla f(x_i))^T (\sqrt{1-M} \odot x_i))^2
    \\
    &= \frac{1}{2m}\sum_{i=1}^m |g(f_\theta(x_i))(1-g(f_\theta(x_i)))| \bE_{\lambda \sim \tilde{D}_\lambda, M} s(M, x_i)^2 \norm{\nabla f(x_i)}^2_2 \norm{x_i}_2^2
    \\
    &\geq \frac{\min_i\norm{x_i}^2 \min_i \bE_{\lambda \sim \tilde{D}_\lambda, M} s(M, x_i)^2 }{2m}\sum_{i=1}^m |g(f_\theta(x_i))(1-g(f_\theta(x_i)))|  \norm{\nabla f(x_i)}^2_2,
\end{align*}
which concludes the theorem.
\end{proof}
\section{Generalization properties of MSDA}
\label{appendix::thm3pf}
The data-dependent MSDA regularization can be altered to the original empirical risk minimization problem with a constrained function set. The Rademacher complexity of this constrained function set is $\cO(1/\sqrt{n})$, which leads to the generalization properties of MSDA. We investigate two models. The first is the GLM model, which has the loss function $l(\theta, z) = A(\theta^\intercal x) - y\theta^\intercal x$. The second is two-layer ReLU networks, which can be parameterized as $ f_\theta(x) = \theta_1^\intercal \sigma(Wx) + \theta_0$. In this case, we consider the mean square error (MSE) loss function (\ie, $l(\theta) = \frac{1}{m} \sum_{i=1}^m (y_i - f_\theta(x_i))^2$)

\subsection{GLM Model}
\label{appendix::glm}
For GLM, using \eqref{eqn::thm1-reformulation}, since the prediction of the GLM model is invariant to the scaling of the training data, we think the dataset $\hat{D} = \{\hat{z_i}\}_{i=1}^m$ with $\hat{x_i} = 1\oslash \bar{M} \odot (M \odot x_i + (1-M) \odot r_x)$ where $\bar{M} = \bE M$. Then, the loss function is 
\begin{align*}
    L_{m}^{\text{(MSDA)}} = \frac{1}{m}\bE_{\lambda}\bE_{r_x} \bE_{M} \sum_{i=1}^m l(\theta, \tilde{z_i}) = \frac{1}{m}\bE_{\xi}\sum_{i=1}^m(A(\hat{x_i}^\intercal \theta) - y_i \hat{x_i}^\intercal \theta),
\end{align*}
where $\xi$ denotes the randomness of $\lambda, r_x$, and $M$. By the second approximation of $A(\cdot)$, we can express $A(\hat{x_i}^\intercal\theta)$ as
\begin{align*}
  A(\hat{x_i}^\intercal\theta) =   A(x_i^\intercal \theta) + A'(x_i^\intercal \theta)(\hat{x_i} - x_i)^\intercal \theta  + \frac{1}{2} A''(x_i^\intercal \theta) \theta^\intercal (\hat{x_i} - x_i)(\hat{x_i} - x_i)^\intercal \theta
\end{align*}
to approximate the loss function. Therefore, we have 
\begin{align}
    \tilde{L}_{m}^{\text{(MSDA)}} &= \frac{1}{m}\sum_{i=1}^m A({x_i}^\intercal \theta) + \frac{1}{m} \bE_{\xi}\sum_{i=1}^m\left(A'(x_i^\intercal \theta)(\hat{x_i} - x_i)^\intercal \theta  + \frac{1}{2} A''(x_i^\intercal \theta) \theta^\intercal (\hat{x_i} - x_i)(\hat{x_i} - x_i)^\intercal \theta\right) \nonumber
    \\
    &= \frac{1}{m}\sum_{i=1}^m A({x_i}^\intercal \theta) + \frac{1}{m} \sum_{i=1}^m\left(\frac{1}{2} A''(x_i^\intercal \theta) \theta^\intercal \text{Var}_{\xi}(\hat{x_i})\theta\right), \label{eqn::thm3-robustness-glm-original}
\end{align}
where $\tilde{L}_{m}^{\text{(MSDA)}}$ denotes the approximate loss of $L_{m}^{\text{(MSDA)}}$ since $\bE_\xi r_x = 0$ and $\bE_\xi \hat{x_i} = x_i$. 
For calculating $\text{Var}_{\xi}(\hat{x_i})$, we use the law of total variance. We have 
\begin{align*}
    \text{Var}_\xi (\tilde{x_i}) &= \left(\frac{1}{\bar{M}}\frac{1}{\bar{M}}^\intercal\right)\odot \text{Var}_\xi \left(M \odot x_i + (1-M) \odot r_x\right)
    \\
    &=  \left(\frac{1}{\bar{M}}\frac{1}{\bar{M}}^\intercal\right)\odot \left( \bE (\text{Var} \left(M \odot x_i + (1-M) \odot r_x\,|\, \lambda, M \right) +\text{Var}(\bE   \left(M \odot x_i + (1-M) \odot r_x\,|\, \lambda, M \right) \right)
    \\
    &= \left(\frac{1}{\bar{M}}\frac{1}{\bar{M}}^\intercal\right)\odot \left( \bE (1-M) \hat{\Sigma}_X (1-M)^\intercal  + x_i \text{Var}(M)  x_i^\intercal )\right)
    \\ 
    &= \frac{1}{\bar{\lambda}^2}\left( \bE (1-M) \hat{\Sigma}_X (1-M)^\intercal  + x_i \text{Var}(M)  x_i^\intercal )\right), 
\end{align*}
where $\hat{\Sigma}_X = \frac{1}{m} \sum_{i=1}^m x_i x_i^\intercal$ with some notational ambiguity that $\frac{1}{\bar{M}}: = \vec{1} \oslash \bar{M}$. In our setting $\bar{M} = \bar{\lambda} \vec{1}$ where $\bar{\lambda} = \bE_{\lambda \sim \tilde{\cD}_\lambda}[\lambda]$. Now we think the related dual problem to the \eqref{eqn::thm3-robustness-glm-original}: 
\begin{align*}
    &\cW_\gamma = \Biggl\{ x \to \theta^\intercal x, \text{ such that } \theta \text{ satisfying }
    \\
    &\left(\bE_x A''(\theta^\intercal x )\right) \cdot \left(\theta^\intercal \left(\bE (1-M) \Sigma_X  (1-M)^\intercal \right) \theta + \theta^\intercal \left( \left(x \text{Var}(M)x^\intercal\right)\right) \theta \right)\leq \gamma \Biggr\}.
\end{align*}
Here, we assume the $\left(\bE_x[A''(v^\intercal x)] \right)^2 \geq \rho \bE_x(v^\intercal x)^2$, which is called $\rho$-retentiveness \cite{zhang2020does, arora2021dropout}. 

\textbf{Theorem 4-(a)} (Restated)\textbf{.} \emph{
\label{thm::GLM-MSDA-generalization-re}
Define $\Sigma_X^{(M)}=  \bE (1-M) \Sigma_X  (1-M)^\intercal$. Suppose $A(\cdot)$ is $L_A$ Lipschitz, and $\cX, \cY, \Theta$ are all bounded. There exist constants $L, B >0$, such that for all $\theta$ that $\theta^\intercal x \in \cW_{\gamma}$, which is the regularization induced by MSDA, we have 
\begin{align*}
    L(\theta) \leq L_m(\theta) + 2LL_A\frac{1}{\sqrt{n}}  (\gamma/\rho)^{1/4} \left(\sqrt{\text{tr}\left(\left(\Sigma_X^{(M)}\right)^{\dagger} \Sigma_X\right)}+ \text{rank}(\Sigma_X) \right) + B \sqrt{\frac{\log(1/\delta)}{2n}},
\end{align*}
with probability at least $1-\delta$.
}

\begin{proof}
Firstly, we calculate the empirical Rademacher complexity of $\cW_\gamma$. For $n$ i.i.d. Rademacher random variables $\xi_1, \dots, \xi_n$, the definition of the empirical Rademacher complexity gives 
\begin{align*}
    \text{Rad}(\cW_\gamma, n) &= \bE_{\xi_i} \sup_{\left(\bE_x A''(\theta^\intercal x )\right) \cdot \left(\theta^\intercal \Sigma_X^{(M)}\theta + \theta^\intercal \left(  \left(x \text{Var}(M)x^\intercal\right)\right) \right) \theta \leq \gamma} \frac{1}{n} \sum_{i=1}^n \xi_i \theta^\intercal x_i
    \\
    &\leq \bE_{\xi_i} \sup_{\left(\bE_x A''(\theta^\intercal x )\right) \cdot \theta^\intercal \Sigma_X^{(M)} \theta\leq \gamma} \frac{1}{n} \sum_{i=1}^n \xi_i \theta^\intercal x_i.
\end{align*}
Due to the $\rho$-retentiveness, we have 
\begin{align*}
    \text{Rad}(\cW_\gamma, n) &\leq \bE_{\xi_i} \sup_{\left(\theta^\intercal \Sigma_X \theta\right) \cdot \left(\theta^\intercal \Sigma_X^{(M)} \theta \right) \leq \gamma/\rho} \frac{1}{n} \sum_{i=1}^n \xi_i \theta^\intercal x_i
    \\
    &\leq \bE_{\xi_i} \left( \sup_{\theta^\intercal \Sigma_X \theta \leq \sqrt{\gamma/\rho}}  \frac{1}{n} \sum_{i=1}^n \xi_i \theta^\intercal x_i +  \bE_{\xi_i}\sup_{ \theta^\intercal \Sigma_X^{(M)} \theta\leq \sqrt{\gamma/\rho}} \frac{1}{n} \sum_{i=1}^n \xi_i \theta^\intercal x_i\right)  
\end{align*}
For the first part of the RHS, define $\tilde{x}_i =  \Sigma_X^{\dagger/2}x_i$ and $v = \Sigma_X^{1/2} \theta$. Then, we have 
\begin{align*}
     \bE_{\xi_i}\sup_{\theta^\intercal \Sigma_X \theta \leq \sqrt{\gamma/\rho}}  \frac{1}{n} \sum_{i=1}^n \xi_i \theta^\intercal x_i &= \bE_{\xi_i} \sup_{\norm{v}^2 \leq \sqrt{\gamma/\rho}}  \frac{1}{n} \sum_{i=1}^n \xi_i v^\intercal \tilde{x}_i
     \\
     &\leq \frac{1}{n}  (\gamma/\rho)^{1/4} \bE_{\xi_i} \norm{\sum_{i=1}^n \xi_i \tilde{x}_i} \leq  \frac{1}{n}  (\gamma/\rho)^{1/4} \sqrt{\bE_{\xi_i} \norm{\sum_{i=1}^n \xi_i \tilde{x}_i}^2}
     \\
     &= \frac{1}{n}  (\gamma/\rho)^{1/4} \sqrt{\sum_{i=1}^n \tilde{x}_i^\intercal \tilde{x}_i}. 
\end{align*}
Similarly, by defining $\check{x}_i =  \left(\Sigma_X^{(M)}\right)^{\dagger/2}x_i$ and $v = \left(\Sigma_X^{(M)}\right)^{1/2} \theta$,  
\begin{align*}
     \bE_{\xi_i}\sup_{ \theta^\intercal \Sigma_X^{(M)} \theta\leq \sqrt{\gamma/\rho}} \frac{1}{n} \sum_{i=1}^n \xi_i \theta^\intercal x_i &= \bE_{\xi_i} \sup_{\norm{v}^2 \leq \sqrt{\gamma/\rho}}  \frac{1}{n} \sum_{i=1}^n \xi_i v^\intercal \check{x}_i
     \\
     &\leq \frac{1}{n}  (\gamma/\rho)^{1/4} \bE_{\xi_i} \norm{\sum_{i=1}^n \xi_i \check{x}_i} \leq  \frac{1}{n}  (\gamma/\rho)^{1/4} \sqrt{\bE_{\xi_i} \norm{\sum_{i=1}^n \xi_i \check{x}_i}^2}
     \\
     &= \frac{1}{n}  (\gamma/\rho)^{1/4} \sqrt{\sum_{i=1}^n \check{x}_i^\intercal \check{x}_i}. 
\end{align*}
Therefore, 
\begin{align*}
    \text{Rad}(\cW_\gamma) &= \bE [\text{Rad}(\cW_\gamma, n)] \leq   \frac{1}{n}  (\gamma/\rho)^{1/4} \left(\sqrt{\sum_{i=1}^n \bE_{x} \tilde{x}_i^\intercal \tilde{x}_i} + \sqrt{\sum_{i=1}^n \bE_{x} \check{x}_i^\intercal \check{x}_i} \right) 
    \\
    & \leq \frac{1}{\sqrt{n}}  (\gamma/\rho)^{1/4} \left(\sqrt{\text{tr}\left(\left(\Sigma_X^{(M)}\right)^{\dagger} \Sigma_X\right)}+ \text{rank}(\Sigma_X) \right).
\end{align*}
The relationship between Rademacher complexity and generalization error \cite{bartlett2002rademacher} indicates Theorem~\ref{thm::MSDA-generalization}. 
\end{proof}

\subsection{Two-layer ReLU Networks}
We perform MSDA on the final layer of the two-layer ReLU networks. Therefore, the setting is the same as Appendix~\ref{appendix::glm} with covariates $\sigma(w_j^\intercal x)$. Due to the scaling of $\theta_1$ and $\theta_0$, we consider training $\theta_1$ and $\theta_0$ on the covariates $1 \oslash \bar{M} \odot( M \odot (\sigma(Wx_i) - \bar{\sigma}_W) + (1-M) \odot (\sigma(Wr_x) - \bar{\sigma}_W))$ where $\bar{\sigma}_W = \frac{1}{n}\sum_{i=1}^m \sigma(Wx_i)$. Putting GLM loss with $A(\cdot) = \frac{1}{2}(\cdot)^2$, 
we have 
\begin{align}
    \tilde{L}_{m}^{\text{(MSDA)}} &= \frac{1}{m}\sum_{i=1}^m (f_\theta(x_i) - y_i)^2  + \frac{1}{m} \sum_{i=1}^m\left(\frac{1}{2}  \theta^\intercal \text{Var}_{\xi}(\sigma(W(M \odot x_i + (1-M) \odot r_x))\theta\right) \label{eqn::thm4-robustness-relu-original}
\end{align}
where $\tilde{L}_{m}^{\text{(MSDA)}}$ denotes the approximate loss of $L_{m}^{\text{(MSDA)}}$ and 
\begin{align*}
    \text{Var}_\xi (\tilde{x_i}) &= \frac{1}{\bar{\lambda}^2} \left( \bE (1-M) \hat{\Sigma}_X^{\sigma} (1-M)^\intercal  + \sigma(Wx_i) \text{Var}(M)  \sigma(Wx_i)^\intercal )\right)
\end{align*}
where $\hat{\Sigma}_X^{\sigma} = \text{Var}_{r_x \sim \cD_X} \sigma(Wr_x)$ with some notational ambiguity that $\frac{1}{\bar{M}}: = \vec{1} \oslash \bar{M}$. Now we think of the related dual problem to the equation \ref{eqn::thm4-robustness-relu-original}: 
\begin{align*}
    &\cW_\gamma^{\text{NN}} = \Biggl\{ x \to f_\theta(x) = \theta_1^\intercal \sigma(Wx) + \theta_0, \text{ such that } \theta \text{ satisfying }
    \\
    &\theta_1 ^\intercal \left( \bE (1-M) \Sigma_X^{\sigma} (1-M)^\intercal \right) \theta_1 + \theta_1 ^\intercal \left( \bE_x  \left(\sigma(Wx) \text{Var}(M)\sigma(Wx)^\intercal\right)\right) \theta_1 \leq \gamma \Biggr\},
\end{align*}
where $ \Sigma_X^{\sigma}  = \text{Var}_x  \sigma(Wx)$.  

\textbf{Theorem 4-(b)} (Restated)\textbf{.} \emph{
\label{thm::GLM-mfdMSDA-generalization-re}
Define $\Sigma_X^{\sigma, (M)}=  \bE (1-M) \Sigma_X^{\sigma}  (1-M)^\intercal$. Suppose $\cX, \cY, \Theta$ are all bounded. There exists constants $L, B >0$, such that for all $\theta$ that $f_\theta(x) \in \cW_{\gamma}^{\text{NN}}$, which is the regularization induced by Manifold MSDA, we have 
\begin{align*}
    L(\theta) \leq L_m(\theta) + 4L\sqrt{\frac{\gamma \left(\text{rank}(\Sigma_X^{\sigma, (M)})+ \norm{\left(\Sigma_X^{\sigma, (M)}\right)^{\dagger/2}\mu_\sigma}^2\right)}{n}} + B \sqrt{\frac{\log(1/\delta)}{2n}},
\end{align*}
with probability at least $1-\delta$.
}

\begin{proof}
Firstly, we calculate the empirical Rademacher complexity of $\cW_\gamma^{\text{NN}}$. For the $n$ i.i.d. Rademacher random variables $\xi_1, \dots, \xi_n$, the definition of the empirical Rademacher complexity gives 
\begin{align*}
    \text{Rad}(\cW_\gamma^{\text{NN}}, n) &= \bE_{\xi_i} \sup_{\theta_1 ^\intercal \left( \bE (1-M) \Sigma_X^{\sigma} (1-M)^\intercal \right) \theta_1 + \theta_1 ^\intercal \left( \bE_x  \left(\sigma(Wx) \text{Var}(M)\sigma(Wx)^\intercal\right)\right) \theta_1 \leq \gamma} \frac{1}{n} \sum_{i=1}^n \xi_i \theta_1^\intercal \sigma(Wx_i)
    \\
    &\leq \bE_{\xi_i} \sup_{\theta_1 ^\intercal \left( \bE (1-M) \Sigma_X^{\sigma} (1-M)^\intercal \right) \theta_1 \leq \gamma} \frac{1}{n} \sum_{i=1}^n \xi_i \theta_1^\intercal \sigma(Wx_i)
    \\
    &\leq \bE_{\xi_i} \sup_{\theta_1 ^\intercal \left( \bE (1-M) \Sigma_X^{\sigma} (1-M)^\intercal \right) \theta_1 \leq \gamma} \frac{1}{n} \sum_{i=1}^n \xi_i \theta_1^\intercal (\sigma(Wx_i)-\mu_\sigma) +\bE_{\xi_i} \sup_{\theta_1 ^\intercal \left( \bE (1-M) \Sigma_X^{\sigma} (1-M)^\intercal \right) \theta_1 \leq \gamma} \frac{1}{n} \sum_{i=1}^n \xi_i \theta_1^\intercal \mu_\sigma,
\end{align*}
where $\mu_\sigma = \bE[\sigma(Wx)]$. Setting $\tilde{\theta_1}^\intercal = \left(\Sigma_X^{\sigma, (M)}\right)^{1/2}$, same technique with the proof of Theorem 3-(a) gives 
\begin{align*}
    \text{Rad}(\cW_\gamma^{\text{NN}}) \leq 2 \sqrt{\frac{\gamma \left(\text{rank}(\Sigma_X^{\sigma, (M)})+ \norm{\left(\Sigma_X^{\sigma, (M)}\right)^{\dagger/2}\mu_\sigma}^2\right)}{n}}. 
\end{align*}
Finally, the relationship between Rademacher complexity and generalization error \cite{bartlett2002rademacher} indicates Theorem 4. 
\end{proof}

\section{Extending Chidambaram \etal \cite{chidambaram2021towards}}
\label{appendix:chidambaram}
Chidambaram \etal \cite{chidambaram2021towards} gave the theoretical analysis of Mixup. We follow this paper by using the unified framework that we used. By modifying several definitions, we get similar results with Chidambaram \etal \cite{chidambaram2021towards}.

Here, we assume that $k$ classes have disjoint support, \ie $X = \bigcup_{i=1}^k X_i$ and $X_i$ are mutually disjoint for $i=1, \dots, k$, where $X_i$ is the support of the $i$th class. We consider cross-entropy loss. We define an associated probability measure $\bP_X$. Then, $L^{\text{MSDA}}$, which is the expected loss for MSDA, can be expressed with 
\begin{align*}
    L^{\text{(MSDA)}}(\theta) = \bE_{z_1, z_2 \sim X}{\bE_{\lambda \sim \cD_\lambda} \bE_M l(\theta, \msda{z}{z_1,z_2}(\lambda, 1-\lambda))},
\end{align*}
where $l$ is the cross entropy function. We can express $L^{\text{(MSDA)}}(\theta) = \sum_{i=1}^k \sum_{j=1}^k L^{\text{(MSDA)}}_{i,j}(\theta)$ with $i,j \in [k]$, where $L^{\text{(MSDA)}}_{i,j}(\theta)$ is defined as 
\begin{align*}
    L^{\text{(MSDA)}}_{i,j}(\theta) = \bE_{z_1, z_2 \sim X}\bE_{\lambda \sim \cD_\lambda} \bE_M \left[l(\theta, \msda{z}{z_1,z_2}(\lambda, 1-\lambda)) I(z_1 \in X_i, z_2 \in X_j)\right].
\end{align*}
$L^{\text{(MSDA)}}_{i,j}(\theta)$ is the full MSDA cross entropy loss corresponding to the mixing points from classes $i$ and $j$. The goal of standard training is to learn a classifier $h \in \argmin_{g \in \cC} L(g,  \bP_X)$ where $\cC$ is the classifier set. Any such classifier $h$ will satisfy $h(x)_i = 1$ on $X_i$ since the $X_i$ are disjoint. 

We modify some definitions in \cite[Section 2.2]{chidambaram2021towards}. We define $A_{x, \epsilon}^{i,j}$ and $A_{x, \epsilon,\delta}^{i,j}$ as 
\begin{align*}
    A_{x, \epsilon}^{i,j} &= \{ (s, t, \lambda, M) \in X_i \times X_j \times [0,1] \times \bR^n : M \odot s + (1- M) \odot t \in B_{\epsilon}(x)\}
    \\
    A_{x, \epsilon, \delta}^{i,j} &= \{ (s, t, \lambda, M) \in X_i \times X_j \times [0,1-\delta] \times \bR^n : M \odot s + (1- M) \odot t \in B_{\epsilon}(x)\}    
    \\
    X_{\text{MSDA}} &= \left\{ x \in \bR^n: \bigcup_{i,j} A_{x, \epsilon}^{i,j} \text{ has positive measure for every }\epsilon > 0 \right\}
    \\
    \xi_{x, \epsilon}^{i,j} &= \bE_{z_1, z_2 \sim X} \bE_{\lambda \sim \cD_\lambda} \bE_M [I(z_1 \in X_i, z_2 \in X_j)]
    \\
    \xi_{x, \epsilon, \lambda}^{i,j} &= \bE_{z_1, z_2 \sim X} \bE_{\lambda \sim \cD_\lambda} \bE_M [\lambda I(z_1 \in X_i, z_2 \in X_j)].
\end{align*}
\begin{definition}[\cite{chidambaram2021towards}]
Let $\cC^{*}$ to be the subset of $\cC$ for which every $h \in \cC^{*}$ satisfies $h(x) = \lim_{\epsilon \to 0} \argmin_{\theta \in [0,1]} L^{\text{(MSDA)}}(\theta)|_{B_\epsilon(x)}$ for all $x \in X_{\text{MSDA}}$ when the limit exists. Here, $L^{\text{(MSDA)}}(\theta)|_{B_\epsilon(x)}$ represents the MSDA loss for a constant function with value $\theta$ with the restriction of each term in $L^{\text{(MSDA)}}$ to the set $A_{x, \epsilon}^{i,j}$
\end{definition}
$\cC^{*}$ includes deep neural networks. Below lemmas and theorems can be proved with the same technique as Chidambaram \etal.
\begin{lemma}
Any function $h \in \argmin_{g \in \cC^{*}}L^{\text{(MSDA)}}(g, \bP_X, \bP_f)$ satisfies $L^{\text{(MSDA)}}(h) \leq L^{\text{(MSDA)}}(g)$ for any continuous $g \in \cC$ 
\end{lemma}
\begin{theorem}
For any point $x \in X_{\text{MSDA}}$ and $\epsilon >0$, there exists a continuous function $h_\epsilon$ satisfying 
\begin{align*}
    h^i_\epsilon(x) = \frac{\xi_{x, \epsilon}^{i,i}+({\sum_{j \neq i} \xi_{x, \epsilon, \lambda}^{i,j}} + (\xi_{x, \epsilon}^{j,i} - \xi_{x, \epsilon, \lambda}^{j,i}))}{\sum_{q=1}^k \xi_{x, \epsilon}^{q,q} + \sum_{j \neq q} (\xi_{x, \epsilon, \lambda}^{q,j} + (\xi_{x, \epsilon}^{j,q} - \xi_{x, \epsilon, \lambda}^{j,q}))},
\end{align*}
and its limits exist when $\epsilon \to 0$.
\end{theorem}
We give an assumption: for a point $x \in X_{\text{MSDA}}$, there exists a class $i$ that $x$ is closest to $X_i$ for arbitrary MSDA expression of $x$, and  $x$ cannot be expressed by MSDA expression between non-$i$ classes. The formal assumption and its geometric intuition can be found in \cite{chidambaram2021towards}.
\begin{theorem}
\label{thm:close}
If $x$ satisfies the above assumption, with respect to a class $i$, then for every $h \in \argmin_{g \in \cC^{*}} L_{\text{MSDA}}(g)$, we have that $h$ classifies $x$ as the class $i$ and $h$ is continuous at x.
\end{theorem}
\cref{thm:close} indicates that if we observe a new sample that can satisfy the above assumption with the class $i$, which is also a distance up to $\min_j d(X_i, X_j)/2$ from the class $i$, the model trained with MSDA will classify it as $i$. Therefore, \cref{thm:close} is closely related to the generalization properties of MSDA. All proof of this section is identical to Chidambaram \etal.

\section{Experimental Details}
\label{appendix::experimental}

\paragraph{Regularized input gradients experiments (Figure 4)}
We investigate the amount of the regularized input gradients by $|\partial_v f_\theta(x) \partial_{v+p} f_\theta(x)|$ with respect to the pixel distance vector $p$. 
We then compute the partial gradients for the validation images $x$, which have not been seen during training, and normalize them to sum to $1$ as 
\begin{equation*}
    \text{PartialGradProd}(x,p) = \max_{v} |\partial_v f_\theta (x) \partial_{v+p} f_\theta (x)|
\end{equation*}
for different $f_\theta$ trained by different MSDA methods.
Finally, we visualize the pixel-wise maximum values of $\text{PartialGradProd}(x,p)$ for the validation images $x$ in Figure 4. 
We train ResNet-50 \cite{he2016deep} models on the ImageNet-1K dataset with $64 \times 64$ image size. 
We show additional visualization in \cref{fig:appendix_partial_grad}, where the $x$-axis denotes the pixel indices sorted by the pixel distance $\norm{p}$.
As shown in Figure 4 and \cref{fig:appendix_partial_grad}, CutMix more regularizes the partial gradient products $|\partial_v f_\theta(x) \partial_{v+p} f_\theta(x)|$ for a closer $p$ than Mixup.  

\begin{figure}[t]
\centering
\includegraphics[width=0.7\textwidth]{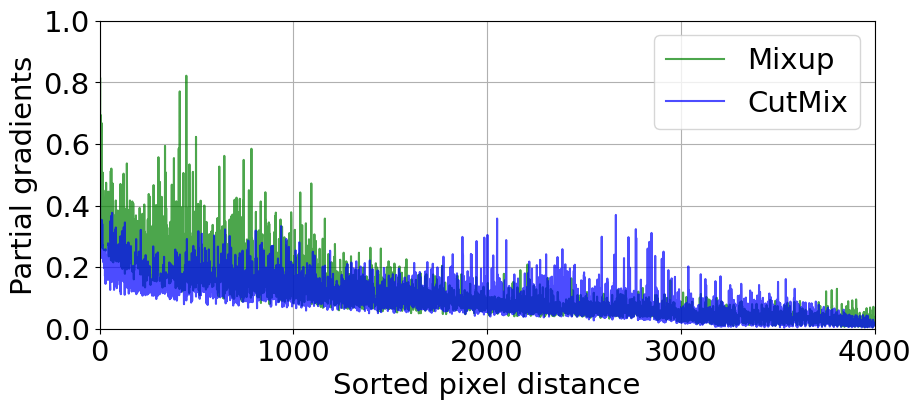}
\caption{Comparison of regularized partial gradients between Mixup and CutMix along the sorted pixel indices. }
\label{fig:appendix_partial_grad}
\end{figure}

\paragraph{CIFAR-100 experiments}
We follow the experimental setting of \cite{kim2020puzzle} and utilize \cite{kim2020puzzle}'s codebase\footnote{https://github.com/snu-mllab/PuzzleMix}. 
We train all networks for 300 epochs using the SGD optimizer with a learning rate of $0.2$ and a batch size of $100$. The learning rate is decayed by a factor of $0.1$ at $100$ and $200$ epochs. 
We set the hyper-parameter $\alpha$ for Mixup, CutMix, and \stochmixupcutmix to $1$.
$\alpha$ for \ours and \oursgaussian were set to $1$ and $0.5$, respectively. 
We use $r=0.5$ for \ours.
All the experiments in CIFAR-100 were repeated three times, and we report the average accuracy.

\paragraph{ImageNet-1K experiments}
We utilize \textit{timm}\footnote{https://github.com/rwightman/pytorch-image-models} pytorch codebase.
We train ResNet-50 \cite{he2016deep} for 300 epochs using the SGD optimizer with a learning rate of  $0.1$, weight decay of  $2\times10^{-5}$, and batch size of $512$. 
We use cosine learning rate scheduling. 
We set the hyper-parameter $\alpha$ for all methods except Mixup to $1$, while Mixup has $\alpha=0.8$.
We use $r=0.75$ for \ours.

\revision{
\section{Additional Experiments}
\label{appendix::additional-experiments}
\paragraph{\ours in the scenarios of Table~\ref{table:ablation}.}
}

{
\begin{table}[t]
\tabcolsep=0.1cm
\centering
\small
\caption{\revision{\small \textbf{\ours performances in both scenarios of Table~\ref{table:ablation}.}}}
\label{table:appendix:scenarios}
\vspace{.5em}
\begin{tabular}{lccccc}
\toprule
                                   & {Mixup}       & {CutMix}     & $\Delta$ (CutMix - Mixup) & \revision{\ours ($r$=0.5)} & \revision{\ours ($r$=0.75)} \\ \midrule
{Scenario 1: Large crop}                       &   58.3                    &    \textbf{64.4}      &  \textcolor{darkergreen}{\textbf{+6.1}}        &     \revision{\textbf{61.3} (+\textbf{3.0} vs. Mixup)} & \revision{\textbf{63.7} (+\textbf{5.4} vs. Mixup)}    \\ 
{Scenario 2: Small crop}                       &   \textbf{67.7}                    &     67.0     &  \textcolor{red2}{\textbf{-0.7}}           &     \revision{\textbf{67.6} (+\textbf{0.6} vs. CutMix)} & \revision{\textbf{67.2} (+\textbf{0.2} vs. CutMix)} \\ 
\bottomrule
\end{tabular}
\vspace{-1em}
\end{table}
}

\revision{
We conduct \ours on two scenarios in Table~\ref{table:ablation}. Results are in Table~\ref{table:appendix:scenarios}. We use $r=0.5$ and $r=0.75$ for \ours. 
\ours shows consistently better performance than Mixup and CutMix for the scenarios 1 and 2, respectively, with meaningful performance gaps. Through the results, we can empirically confirm our proposed method enjoyed the advantages of Mixup and CutMix. 

\paragraph{Robustness benchmarks.}

\begin{table}[t]
\centering
\small
\caption{\revision{\small \textbf{Robustness benchmarks.} Comparison of various MSDA methods on ResNet-50 architecture.}}
\label{table:appendix:robustness}
\vspace{.5em}
\begin{tabular}{lcccc}
\toprule
Augmentation Method             & ImageNet-1K             & ImageNet-occ & ImageNet-C & FGSM \\ \midrule
Vanilla (no MDSA)               &		75.68 (+0.00)           & 55.26 (+0.00) & 42.57 (+0.00) & 8.55 (+0.00)\\
Mixup	                        &		77.78 (+2.10)           & 60.34 (+5.08) & \textbf{51.73 (+9.16)} & 27.78 (+19.23)\\
CutMix	                        &		78.04 (+2.36)           & \textbf{71.51 (+16.25)} & 44.18 (+1.55) & 33.63 (+25.08)\\
\rowcolor{Gray}
\textbf{\ours} (ours)	        &		\textbf{78.38} (+2.70)  & \textbf{71.13 (+15.87)} & \textbf{46.37 (+3.80)} & \textbf{34.98 (+26.44)} \\
\rowcolor{Gray}
\textbf{\oursgaussian} (ours)	&		{78.13} (+2.45)         & 62.76 (+7.50) & 45.97 (+3.40) & 31.02 (+21.47) \\ \bottomrule
\end{tabular}
\end{table}

As observed by the previous study \cite{chun2019icmlw}, the different choice of MSDA methods also affects the extreme case of the test samples, \eg, distribution shifts.
In this experiments, we provide an understanding of the relationship between MSDA and various test scenarios and through the lens of our theoretical analysis.
We conduct various MSDA methods on robustness benchmarks such as ImageNet-1K occlusion accuracy (center occluded images following \cite{yun2019cutmix, chun2019icmlw}), ImageNet-C accuracy \cite{hendrycks2018benchmarking} and adversarially attacked ImageNet test accuracy by FGSM attack \cite{fgsm}.
Results are in Table~\ref{table:appendix:robustness}. 
We use the same experimental settings of the ImageNet classification in Table~\ref{table:imagenet-classification}.
Overall results show that \ours and \oursgaussian are located in between CutMix and Mixup.

CutMix performs better than Mixup in the occlusion accuracy. Here, the local occluded areas have no information to distinguish objects, but other local areas are informative. Hence, it is important to capture shorter-relationship rather than global-relationship. Thus we can expect that CutMix is better than Mixup, and not surprisingly, \ours and \oursgaussian are located in between CutMix and Mixup.

ImageNet-C style corruptions (\eg, adding Gaussian noise for the entire image) distort the local information. In this case, the shorter-distance relationships are significantly damaged and sometimes useless to distinguish the object, hence the longer-distance relationships are more important.
Hence, we can expect Mixup works better than CutMix in ImageNet-C.
As \ours and \oursgaussian less weight shorter-distance relationships than CutMix (but more weight than Mixup), we can observe that \ours and \oursgaussian ImageNet-C performances are better than CutMix, but worse than Mixup.

\begin{table}[t]
\caption{\revision{\small \textbf{Ablation study on hyper-parameters.}}}
\begin{subtable}[h]{0.3\textwidth}
\tabcolsep=0.06cm
\centering
\small
\caption{\revision{\small \textbf{Impact on $r$ for \ours.}}}
\label{table:appendix:ablationa}
\vspace{.5em}
\begin{tabular}{lc}
\toprule
PreActRN18                          &  CIFAR-100 acc \\ \midrule
Vanilla (no MDSA)               	&    76.73        \\
Mixup	($r=0$)                 	&      77.21         \\
CutMix	($r=1.0$)                    &       78.66       \\
\textbf{\ours} ($r=0.75$) 	        	&   \textbf{79.43}        \\
\rowcolor{Gray}
\textbf{\ours} ($r=0.5$) 	        	&    79.25       \\
\textbf{\ours} ($r=0.25$) 	        	&      78.05     \\\bottomrule
\end{tabular}
\end{subtable}
\hfill
\begin{subtable}[h]{0.3\textwidth}
\tabcolsep=0.06cm
\centering
\small
\caption{\revision{\small \textbf{Impact on $\alpha$ for \ours.}}}
\label{table:appendix:ablationb}
\vspace{.5em}
\begin{tabular}{lc}
\toprule
PreActRN18                          &  CIFAR-100 acc \\ \midrule
Vanilla (no MDSA)               	&     76.73        \\
Mixup	($\alpha=1.0$)                 	&      77.21         \\
CutMix	($\alpha=1.0$)                    &       78.66       \\
\textbf{\ours} ($\alpha=0.5$) 	        	&   {78.47}        \\
\rowcolor{Gray}
\textbf{\ours} ($\alpha=1.0$) 	        	&    \textbf{79.25}       \\
\textbf{\ours} ($\alpha=2.0$) 	        	&      78.85     \\\bottomrule
\end{tabular}
\end{subtable}
\hfill
\begin{subtable}[h]{0.3\textwidth}
\tabcolsep=0.06cm
\centering
\small
\caption{\revision{\small \textbf{Impact on $\alpha$ for \oursgaussian}}}
\label{table:appendix:ablationc}
\vspace{.5em}
\begin{tabular}{lc}
\toprule
PreActRN18                          &  CIFAR-100 acc \\ \midrule
Vanilla (no MDSA)               	&     76.73        \\
Mixup	($\alpha=1.0$)                 	&      77.21         \\
CutMix	($\alpha=1.0$)                    &       78.66       \\
\textbf{\oursgaussian} ($\alpha=0.25$) 	        	&   {78.60}        \\
\rowcolor{Gray}
\textbf{\oursgaussian} ($\alpha=0.5$) 	        	&    \textbf{79.17}       \\
\textbf{\oursgaussian} ($\alpha=0.75$) 	        	&      78.64     \\
\textbf{\oursgaussian} ($\alpha=1.0$) 	        	&      79.05     \\\bottomrule
\end{tabular}
\end{subtable}
\label{table:appendix:ablation}
\end{table}

\paragraph{Ablation study on hyper-parameters.}

We study the impacts of hyper-parameters $\alpha$ and $r$ for \ours and \oursgaussian on CIFAR-100 classification.
We use the PreActResNet-18 with the same experimental setting as in Table~\ref{table:cifar-classification}.
Results are in Table~\ref{table:appendix:ablation}.
We highlight the results with our hyper-parameter choices ($r=0.5$, $\alpha=1.0$ for \ours, $\alpha=0.5$ for \oursgaussian) in the gray cells. 
For $r$, we find that \ours with $r=0.75$ performs better than $r=0.5$ with a marginal gap (+0.18\%p). 
For $\alpha$, our hyper-parameters show the best performance.
Overall results confirm that \ours and \oursgaussian are not sensitive to those hyper-parameters and consistently show better or compatible performance against Mixup and CutMix.

\revision{
\section{Regularizer coefficients of HMix}
\label{appendix::hmix-regularizer}
Define
\begin{align}
    h(x,s) & = \min(x, n-s), \qquad l(x, s) =  \max(x - s, 0), \nonumber
    \\
    \medmath{a_{jk,s}} &= \medmath{\frac{\max(\min(h(j_1,s) - l(k_1,s), h(k_1,s) - l(j_1,s)), 0)  \max(\min(h(j_2,s) - l(k_2,s), h(k_2,s) - l(j_2,s)), 0)}{(n - s)^2}  }, \label{eqn::ours-coef}
    \\
    o_s &= \frac{\lambda n^2}{n^2 - s^2}, \nonumber
    \\
    v(p, s)& =  \frac{(h(p_1, s) - l(p_1, s)) * (h(p_2, s) - l(p_2, s))}{(n-s)^2}. \nonumber
\end{align}

Note that \eqref{eqn::ours-coef} is extension of \eqref{eqn::cutmix-coef} by putting $s = [\sqrt{1-\lambda}n]$ in \eqref{eqn::ours-coef}. Then, HMix with hyperparameter $r$ has regularizer coefficient $a_{ij}$ as 
\begin{align*}
    s &= [\sqrt{1-\lambda}\sqrt{r}n]\\
    a_{ij} &= o(s)(1- o(s))(v(i, s) + v(j,s)) + o(s)a_{ij,s} + (1- o(s))(1 - o(s)).
\end{align*}
We plotted this value in \cref{fig:coefcomp} when $r = 0.7$.

}

\section{What MSDA can be applied in our Thereoms?}
\label{appendix::other-msdas}
Our theorems can be applied to any MSDA method with an analogous formula, regardless of the assumption of the shape of the mask. In this paper, we mainly focused on Mixup and CutMix because they are the most common MSDA methods among the whole MSDA family as well as their behaviors are distinctly different in terms of our theorem. In this section, we note several nontrivial remarks for understanding the setup of our paper. 

\paragraph{ResizeMix.}
ResizeMix \cite{qin2020resizemix} can be explained by our theorem if we add the assumption on the dataset $\mathcal{D}_X$ that $\mathcal{D}_X$ has all resized versions of the image. ResizeMix uses the resized version of input (i.e., one of the mixed patches is the “resized” version, not a cropped one) where the random resize is applied to the whole dataset. In other words, ResizeMix is a special case of CutMix when we apply a special version of random resize crop operation. Hence, if we assume a different version of random resize crop rather than the standard version (independent of our theoretical results and underlying assumptions), ResizeMix is equivalent to CutMix, which leads to the same theoretical result as CutMix.

\paragraph{FMix.}
FMix \cite{harris2020fmix} randomly samples the mask from the Fourier space. Since FMix is one of the static MSDA, we can directly apply our Theorem 1-4. 

\paragraph{SaliencyMix, PuzzleMix, and Co-Mixup.}
SaliencyMix \cite{uddin2020saliencymix} uses saliency map to generate new MSDA sample. PuzzleMix \cite{kim2020puzzle}, and Co-Mixup \cite{kim2021co} are dynamic MSDA, where they use saliency map and transport. These methods give a state-of-the-art performance. \cref{thm::MSDA-loss} can deal with this problem, but hard to interpret. To be specific, the second equality of \cref{eqn::thm1-GH} do not hold anymore; it is hard to interpret the approximated loss function as an input gradient / Hessian regularizer. 

\paragraph{StyleMix and Manifold Mixup.}
StyleMix \cite{hong2021stylemix} uses pre-trained style encoder and decoder. StyleMix linearly mixes content and style. Manifold Mixup \cite{verma2019manifold} mixes samples in the feature level. Therefore, the theorems cannot be directly applied to StyleMix and Manifold Mixup. %

\paragraph{AutoMix.}
Recently, AutoMix \cite{liu2021unveiling} gives state-of-the-art results. AutoMix utilized joint loss to generate the mask $M$: classification loss and generation loss for training $M$. Therefore, the mask depends on the mixing samples, indicating that AutoMix is a dynamic MSDA. As in previous paragraphs, \cref{thm::MSDA-loss} holds, but it is not easy to interpret each term. 

}

\end{document}